\documentclass[lettersize,journal]{IEEEtran}
\usepackage{amsmath,amsfonts}
\usepackage{algorithmic}
\usepackage{algorithm}
\usepackage{array}
\usepackage[caption=false,font=normalsize,labelfont=sf,textfont=sf]{subfig}
\usepackage{textcomp}
\usepackage{stfloats}
\usepackage{url}
\usepackage{verbatim}
\usepackage{graphicx}
\usepackage{cite}
\usepackage{amssymb}
\usepackage{booktabs}
\usepackage{tabularx}
\usepackage{orcidlink}  
\usepackage{microtype}

\newtheorem{proposition}{Proposition}
\newtheorem{assumption}{Assumption}
\newtheorem{theorem}{Theorem}
\usepackage[justification=centering]{caption}
\usepackage{xcolor}
\begin{document}

\title{Credibility-Aware Learning and Control for Safe USV Navigation under Perception Uncertainty}

\author{Yuhang Zhang, Shuqi Chai, Yukang Zhang, Liusha Yang, Mingchuan Zhang, Wei Wang, Qingjiang Shi and Quanbo Ge$^{\orcidlink{0000-0001-6361-880X}}$, \IEEEmembership{Member, IEEE} 
\thanks{This work was supported in part by the National Natural Science Foundation of China (NSFC) under Grants 62033010 and 62401376, in part by the Qing Lan Project of Jiangsu Province of China under Grant R2023Q07, in part by the Shenzhen Science and Technology Program under Grant RCBS20231211090816034, and in part by the Natural Science Foundation of Top Talent of SZTU under Grant GDRC202332. (Corresponding author: Quanbo Ge.)}
\thanks{Yuhang Zhang and Mingchuan Zhang are with the School of Information Engineering, Henan University of Science and Technology, Luoyang 471023, China (e-mail: yhZhang@stu.haust.edu.cn; zhang\_mch@haust.edu.cn).}
\thanks{Shuqi Chai is with the Shenzhen Research Institute of Big Data, Shenzhen 518172, China, also with the Chinese University of Hong Kong, Shenzhen, Shenzhen 518172, China (e-mail: schai@sribd.cn).}
\thanks{Yukang Zhang is with the School of Logistics Engineering, Shanghai Maritime University, Shanghai 200135, China (e-mail: 202430210046@stu.shmtu.edu.cn).}
\thanks{Liusha Yang is with Shenzhen Technology University, Shenzhen 518118, China (e-mail: yangliusha@sztu.edu.cn).}
\thanks{Wei Wang is with the School of Computer Science, Wuhan University, Wuhan 430072, China (e-mail: wangw@whu.edu.cn).}
\thanks{Qingjiang Shi is with the School of Software Engineering, Tongji University, Shanghai 201804, China (e-mail: shiqj@tongji.edu.cn).}
\thanks{Quanbo Ge is with the School of Automation, Nanjing University of Information Science and Technology, Nanjing 210044, China (e-mail: geqb@nuist.edu.cn).}
}

\markboth{IEEE Transactions on Vehicular Technology,~Vol.~XX, No.~XX, Month~2026}%
{IEEE Transactions on Vehicular Technology,~Vol.~XX, No.~XX, Month~2026}


\maketitle

\begin{abstract}

Safe navigation for Unmanned Surface Vehicles (USVs) under the International Regulations for Preventing Collisions at Sea (COLREGs) remains challenging in dynamic maritime environments, especially when perception uncertainty is miscalibrated. Errors in state estimation can produce unreliable belief states that mislead value learning, while logic based on discrete traffic rules can cause abrupt action corrections. To address these challenges, we integrate Credibility-Weighted Value Learning (CWVL) with Covariance- and Recovery-Aware Control Barrier Function Quadratic Programming (CoReCBF-QP). CWVL derives a dynamic trust factor from the discrepancy between the covariance estimated by the filter and empirical error statistics. This factor modulates the critic's heteroscedastic loss and limits overfitting to miscalibrated observations. CoReCBF expands the collision geometry according to uncertainty and incorporates terms for braking and turning recovery. The resulting hyperbolic safety boundary preserves feasible avoidance velocities and supplies the QP safety constraint. A continuous COLREGs-aware reference in the objective promotes starboard maneuvers in Rule 14 head-on and Rule 15 give-way crossing encounters. Simulations show improved robustness in collision avoidance and COLREGs event compliance, achieving an 82.0\% success rate with ten target ships beyond the training range.
\par\noindent Code: \url{https://github.com/zyhang6656/CredNav-USV}
\end{abstract}

\vspace{-0.5\baselineskip}

\section{Introduction}
\label{section:Introduction}

\IEEEPARstart{I}{n} recent years, Unmanned Surface Vehicles (USVs) have been increasingly used for hydrographic surveys, environmental monitoring, and maritime logistics, forming an important class of Maritime Autonomous Surface Ships (MASS)~\cite{10021250}. Safe navigation becomes more difficult in dynamic encounters involving several vessels~\cite{Jeong2024ActiveLearningASV}. Measurements of nearby vessels are affected by sensor noise, occlusion, and environmental disturbances, which introduce uncertainty into their estimated motion~\cite{Meyers2024COLREGsConstrainedChannel}. An avoidance maneuver must also remain within the capabilities of an underactuated vessel and be consistent with the International Regulations for Preventing Collisions at Sea (COLREGs)~\cite{Cho2023RestrictedWaterwaysCOLREGs,li2021path,heiberg2022risk}. Errors in the estimated relative state can therefore change the assessed collision risk and, in turn, the maneuver selected to resolve an encounter.\par

Model-based planning and control remain common in USV collision avoidance. At the local level, the Artificial Potential Field (APF) method provides simple reactive guidance but can become trapped in local minima~\cite{khatib1986real}. The Dynamic Window Approach (DWA) evaluates executable velocities over a short interval to account for local motion constraints, although its reactive search can encounter the same difficulty in complex environments~\cite{fox1997dynamic,Molinos2019DynamicWindow}. Hybrid $A^*$ extends the planning horizon and includes nonholonomic motion in path search, but repeated replanning becomes costly when obstacles move~\cite{xu2024hybrid}. For encounters with moving vessels, the Velocity Obstacle (VO) method uses relative motion to identify velocities that may lead to collision~\cite{fiorini1998motion}. It can, however, produce oscillatory decisions and does not directly optimize control smoothness. Model Predictive Control (MPC) instead optimizes vessel motion over a finite horizon and can incorporate dynamics, safety constraints, and traffic rules~\cite{Wei2023MPCMarineSurvey,Peng2023SafetyCertifiedBerthing,Tsolakis2024TrafficRulesASV,Han2024PotentialFieldMPC}. Doing so requires repeated nonlinear optimization online. Although their formulations differ, these methods rely on predefined motion relations and estimated relative states.

Learning methods provide another way to generate avoidance decisions in dynamic encounters~\cite{10379139,Lou2024BalancedUSVDRL}. Deep Reinforcement Learning (DRL) has been combined with VO information to support collision avoidance in dynamic environments~\cite{han2022reinforcement,Li2025DQNVO,Xie2023DRLVO}. Other approaches learn control commands from observed states using Convolutional Neural Networks (CNNs) and Gated Recurrent Units (GRUs)~\cite{tao2024integrated}, or employ feature-enhanced training for underactuated USVs~\cite{wang2023obstacle}. More recent work has used deep ensembles to estimate epistemic uncertainty~\cite{Zhang2024LagU} and return distributions to represent navigation risk~\cite{huang2025uncertainty}. These uncertainty measures describe the learned policy or return, but they do not directly assess the consistency of the belief state supplied by the estimator.

Under partial observability, the critic learns from an estimated belief state rather than the latent state. The belief may become unreliable when the estimator's error model does not match the physical system~\cite{ge2023novel,ge2024credible,zhang2024one}. Errors in the belief state are then passed to value learning, where they coexist with noise in the value target. Heteroscedastic regression models target noise through a variance that changes with the input~\cite{kendall2017what,Mai2022IVRL}. The predicted variance can reduce the influence of noisy targets, but poor calibration or overestimation may suppress useful samples~\cite{Mai2022IVRL}. Beyond sample weighting, variance-dependent scaling also shapes critic gradients; in Gaussian distributional critics, this dependence contributes to training instability and reward-scale sensitivity~\cite{Duan2025DSACT}. The quality of critic gradients also matters when they are used to construct state perturbations, since unreliable Q estimates or gradients may fail to identify genuinely adverse states~\cite{Zhang2020RobustStatePerturbations}. The target uncertainty represented by the critic and the reliability of the belief input are therefore separate concerns.

Uncertain relative states also make deterministic collision geometry unreliable. Chance-constrained methods require the collision probability, computed from an assumed uncertainty distribution, to remain below a prescribed risk level. This idea has been applied to reciprocal velocity obstacles and predictive collision avoidance~\cite{Gopalakrishnan2017PRVO,Zhu2019ChanceConstrainedMAV}. However, the resulting risk bound depends on the assumed distribution and covariance, and does not reveal when the uncertainty estimate itself is unreliable. Collision Cone Control Barrier Functions (C3BFs) address dynamic obstacles by constructing barrier functions from relative position and velocity~\cite{Thontepu2023C3BFUGV,Tayal2024C3BFUAV,Tayal2026C3BF}. CBF-VO further uses VO information for guidance while enforcing safety through a separate, less restrictive control barrier function constraint~\cite{SanchezRoncero2025VOCBF}. These formulations are not readily transferred to underactuated USVs. Surge thrust and yaw moment do not generally appear at the same relative degree in a barrier constructed directly from relative motion. Unactuated sway and coupled surge, sway, and yaw dynamics further complicate their inclusion in one affine safety constraint. High-Order Control Barrier Functions (HOCBFs) extend barrier constraints to arbitrary relative degrees through a recursive construction~\cite{Xiao2022HOCBF}. For the input structure considered here, applying HOCBFs would require a dynamic extension that places the surge thrust rate, rather than surge thrust itself, at the same derivative order as the yaw moment. This can address part of the mismatch, but HOCBFs introduce auxiliary functions and nested safety constraints, complicating the construction and analysis of safety conditions~\cite{Taylor2022SafeBackstepping,Ong2024RectifiedCBF}. Moreover, treating the collision cone as a strict boundary can exclude velocities that remain recoverable through braking or turning.

{
Collision safety alone does not determine which maneuver should be preferred in a regulated encounter. Existing studies have incorporated COLREGs into reinforcement learning methods that account for risk~\cite{wang2024colregs,10530091}, deep reinforcement learning that uses VO information~\cite{Li2025DQNVO}, and predictive control for vessel motion under navigation rules~\cite{Tsolakis2024TrafficRulesASV}. Many methods identify encounter responsibilities through discrete bearing sectors or binary indicators~\cite{DeVries2022RegulationsAwareMotionPlanning,10707649,Zhou2025CollisionAvoidanceTracking}. These representations are effective for encounter classification, but a maneuver reference derived directly from them can change abruptly as the relative bearing crosses a sector boundary. A continuous representation of COLREGs maneuver preference is therefore better suited to continuous action guidance.
}

{
Under partial observability, uncertainty, collision geometry, and navigation rules jointly affect decision making for underactuated USVs. We formulate this problem as a Partially Observable Markov Decision Process (POMDP) and develop a safe reinforcement learning framework. Belief credibility informs value learning, barrier constraints account for covariance and vessel recovery, and a continuous COLREGs reference guides action correction through quadratic programming. The main contributions are as follows:
}
\begin{enumerate}
\item {To mitigate the impact of unreliable perception on value learning, we propose Credibility-Weighted Value Learning (CWVL). Heteroscedastic critics account for uncertainty in value targets through learned variance estimates, whose poor scaling may destabilize training. Perception errors affect value learning by degrading the reliability of the belief state used by the critic. CWVL evaluates this reliability through a dynamic trust factor grounded in filter credibility. The factor quantifies the agreement between the filter's predicted uncertainty and realized estimation errors. It assigns lower weights to samples associated with unreliable beliefs, thereby reducing their influence on value learning.}

\item {We develop a Covariance- and Recovery-Aware Control Barrier Function (CoReCBF) for underactuated USVs under spatial uncertainty. CoReCBF expands the collision geometry according to state covariance and accounts for braking and turning recovery when defining the safety condition. The resulting hyperbolic safety boundary retains recoverable avoidance velocities and supports action correction through quadratic programming.}

\item {We introduce a continuous COLREGs reference for action correction. A bounded yaw reference continuously represents the starboard maneuver preference in Rule 14 head-on and Rule 15 give-way crossing encounters. During action correction, the quadratic program follows this preference within the admissible safety set defined by CoReCBF constraints.}
\end{enumerate}

{The remainder of this paper is organized as follows. Section~\ref{section:Problem_Statement} formulates USV navigation as a POMDP. Section~\ref{section:credibility_learning} presents belief estimation and CWVL. Section~\ref{section:shielded_decision} introduces CoReCBF, action correction through quadratic programming, the continuous COLREGs reference, and reward design. Section~\ref{section:Experiments} reports the experimental evaluation. Section~\ref{section:Conclusion and future work} concludes the paper.}

\section{Problem Formulation}
\label{section:Problem_Statement}

Autonomous navigation in cluttered maritime environments necessitates simultaneous collision avoidance and adherence to traffic rules. However, practical perception systems inevitably introduce state estimation errors, rendering the environment partially observable. To address the decision-making problem under such perceptual uncertainty, the navigation task is formulated as a discrete-time Partially Observable Markov Decision Process (POMDP) with a continuous latent state space and 
a bounded continuous action space. Formally, the POMDP is defined by the tuple
\begin{equation}
    \mathcal{M} = (\mathcal{S}, \mathcal{A}, \mathcal{T}, \mathcal{R}, \Omega, \mathcal{O}).
\end{equation}
In this formulation, $\mathcal{S} \subseteq \mathbb{R}^{n_s}$ denotes the continuous latent state space. Let $N$ denote the number of target ships. At discrete time step 
$k\in\{0,1,\ldots\}$, the joint latent state is defined as
\begin{equation}
\mathbf{s}_k =
\left[
\left(\mathbf{x}^{\mathrm{os}}_k\right)^{\mathsf T},
\left(\mathbf{x}^{\mathrm{ts}}_{1,k}\right)^{\mathsf T},
\ldots,
\left(\mathbf{x}^{\mathrm{ts}}_{N,k}\right)^{\mathsf T}
\right]^{\mathsf T}
\in \mathcal{S},
\label{eq:state_definition}
\end{equation}
where $\mathbf{x}^{\mathrm{os}}_k$ and $\mathbf{x}^{\mathrm{ts}}_{i,k}$ denote the 
kinematic substates of the own ship and target ship, respectively.
The substate of the own ship is
\begin{equation}
\mathbf{x}^{\mathrm{os}}_k =
\left[
p^{\mathrm{os}}_{x,k},
p^{\mathrm{os}}_{y,k},
\psi^{\mathrm{os}}_k,
u^{\mathrm{os}}_k,
v^{\mathrm{os}}_k,
r^{\mathrm{os}}_k
\right]^{\mathsf T},
\label{eq:os_state}
\end{equation}
where $(p^{\mathrm{os}}_{x,k},p^{\mathrm{os}}_{y,k})$ denotes the position of the own ship in the global horizontal frame, $\psi^{\mathrm{os}}_k$ is the heading angle, and 
$u^{\mathrm{os}}_k$, $v^{\mathrm{os}}_k$, and $r^{\mathrm{os}}_k$ are the surge 
velocity, sway velocity, and yaw rate.
For each target ship $i=1,\ldots,N$, the corresponding substate is
\begin{equation}
\mathbf{x}^{\mathrm{ts}}_{i,k} =
\left[
p^{\mathrm{ts}}_{x,i,k},
p^{\mathrm{ts}}_{y,i,k},
v^{\mathrm{ts}}_{x,i,k},
v^{\mathrm{ts}}_{y,i,k}
\right]^{\mathsf T},
\label{eq:ts_state}
\end{equation}
where $(p^{\mathrm{ts}}_{x,i,k},p^{\mathrm{ts}}_{y,i,k})$ and 
$(v^{\mathrm{ts}}_{x,i,k},v^{\mathrm{ts}}_{y,i,k})$ are its position 
and velocity components in the global horizontal frame. 

The action space $\mathcal{A}$ is defined according to the three-degree-of-freedom (3-DOF) underactuated USV model. Since the sway channel is unactuated, only the surge thrust $\tau_{u}$ and yaw moment $\tau_{r}$ are directly commanded. Accordingly, the action vector at time step $k$ is defined as
\begin{equation}
\mathbf{a}_k =
\left[
\tau_{u,k},
\tau_{r,k}
\right]^{\mathsf T}
\in \mathcal{A},
\label{eq:action_definition}
\end{equation}
which is subject to the following physical actuator bounds:
\begin{equation}
\tau_u^{\min}\le \tau_{u,k}\le \tau_u^{\max}, 
\qquad
|\tau_{r,k}|\le \tau_r^{\max}.
\label{eq:action_bounds}
\end{equation}

The vessel motion is described by a discrete-time model. For the own ship,
define the pose vector and the body-frame velocity vector as
\begin{equation}
\boldsymbol{\eta}^{\mathrm{os}}_k =
\left[
p^{\mathrm{os}}_{x,k},
p^{\mathrm{os}}_{y,k},
\psi^{\mathrm{os}}_k
\right]^{\mathsf T},
\qquad
\boldsymbol{\nu}^{\mathrm{os}}_k =
\left[
u^{\mathrm{os}}_k,
v^{\mathrm{os}}_k,
r^{\mathrm{os}}_k
\right]^{\mathsf T},
\label{eq:os_eta_nu}
\end{equation}
where $\boldsymbol{\eta}^{\mathrm{os}}_k$ denotes the global position and
heading, and $\boldsymbol{\nu}^{\mathrm{os}}_k$ denotes the surge velocity,
sway velocity, and yaw rate in the body-fixed frame. The 3-DOF
kinematics of the own ship in the horizontal plane are
\begin{equation}
\dot{\boldsymbol{\eta}}^{\mathrm{os}}
=
\mathbf{R}(\psi^{\mathrm{os}})
\boldsymbol{\nu}^{\mathrm{os}},
\label{eq:os_kinematics}
\end{equation}
where
\begin{equation}
\mathbf{R}(\psi)=
\begin{bmatrix}
\cos\psi & -\sin\psi & 0\\
\sin\psi & \cos\psi & 0\\
0 & 0 & 1
\end{bmatrix}.
\label{eq:rotation_matrix_3dof}
\end{equation}
The corresponding underactuated 3-DOF dynamics are written as
\begin{equation}
\mathbf{M}\dot{\boldsymbol{\nu}}^{\mathrm{os}}
+
\mathbf{C}(\boldsymbol{\nu}^{\mathrm{os}})
\boldsymbol{\nu}^{\mathrm{os}}
+
\mathbf{D}\boldsymbol{\nu}^{\mathrm{os}}
=
\boldsymbol{\tau},
\label{eq:os_dynamics}
\end{equation}
where $\mathbf{M}$ is the inertia matrix,
$\mathbf{C}(\boldsymbol{\nu}^{\mathrm{os}})$ is the Coriolis and centripetal
matrix, and $\mathbf{D}$ is the damping matrix. Since the own ship is
underactuated in the sway channel, the generalized control input at time
step $k$ is
\begin{equation}
\boldsymbol{\tau}_k =
\left[
\tau_{u,k},
0,
\tau_{r,k}
\right]^{\mathsf T}.
\label{eq:generalized_input}
\end{equation}
With sampling interval $\Delta t$, forward Euler discretization gives
\begin{equation}
\boldsymbol{\eta}^{\mathrm{os}}_{k+1}
=
\boldsymbol{\eta}^{\mathrm{os}}_{k}
+
\Delta t\,
\mathbf{R}(\psi^{\mathrm{os}}_k)
\boldsymbol{\nu}^{\mathrm{os}}_{k},
\label{eq:os_eta_update}
\end{equation}
\begin{equation}
\boldsymbol{\nu}^{\mathrm{os}}_{k+1}
=
\boldsymbol{\nu}^{\mathrm{os}}_{k}
+
\Delta t\,
\mathbf{M}^{-1}
\left[
\boldsymbol{\tau}_k
-
\mathbf{C}(\boldsymbol{\nu}^{\mathrm{os}}_k)
\boldsymbol{\nu}^{\mathrm{os}}_k
-
\mathbf{D}\boldsymbol{\nu}^{\mathrm{os}}_k
\right].
\label{eq:os_nu_update}
\end{equation}
Thus, the own ship state update can be compactly written as
\begin{equation}
\mathbf{x}^{\mathrm{os}}_{k+1}
=
f_{\mathrm{os}}
\left(
\mathbf{x}^{\mathrm{os}}_k,\mathbf{a}_k
\right).
\label{eq:os_compact_update}
\end{equation}

For each target ship, a constant-velocity model is adopted:
\begin{equation}
\mathbf{x}^{\mathrm{ts}}_{i,k+1}
=
\mathbf{F}_{\mathrm{ts}}
\mathbf{x}^{\mathrm{ts}}_{i,k},
\qquad i=1,\ldots,N,
\label{eq:ts_update}
\end{equation}
where
\begin{equation}
\mathbf{F}_{\mathrm{ts}} =
\begin{bmatrix}
1 & 0 & \Delta t & 0\\
0 & 1 & 0 & \Delta t\\
0 & 0 & 1 & 0\\
0 & 0 & 0 & 1
\end{bmatrix}.
\label{eq:ts_transition_matrix}
\end{equation}
During their active motion intervals, the simulated target ships follow straight-line constant-velocity trajectories. This controlled setting separates measurement-level estimation mismatch from changes in target motion.
Combining the updates of the own ship and target ships, the joint multi-vessel
state update is written as
\begin{equation}
\mathbf{s}_{k+1}
=
f
\left(
\mathbf{s}_k,\mathbf{a}_k
\right).
\label{eq:joint_state_update}
\end{equation}

The observation space $\Omega$ contains the measured own ship state and
the noisy measured states of target ships, with the tilde notation
$\tilde{(\cdot)}$ used for noisy measurements of target ships. At time step $k$,
the joint observation vector $\mathbf{o}_k \in \Omega$ is defined as
\begin{equation}
\mathbf{o}_k = 
\left[ 
(\mathbf{o}^{\mathrm{os}}_k)^{\mathsf T}, 
(\mathbf{o}^{\mathrm{ts}}_{1,k})^{\mathsf T}, 
\ldots, 
(\mathbf{o}^{\mathrm{ts}}_{N,k})^{\mathsf T} 
\right]^{\mathsf T}. 
\label{eq:observation_definition}
\end{equation}
In this work, $\mathbf{o}^{\mathrm{os}}_k$ is used as the self-state feedback
for control, whereas the dominant perception uncertainty is associated with
target ship tracking. Specifically, $\mathbf{o}^{\mathrm{os}}_k$ comprises the measured global pose and body-frame velocities of the 3-DOF model of the own ship:
\begin{equation}
\mathbf{o}^{\mathrm{os}}_k = 
\left[ 
p^{\mathrm{os}}_{x,k}, 
p^{\mathrm{os}}_{y,k}, 
\psi^{\mathrm{os}}_k, 
u^{\mathrm{os}}_k, 
v^{\mathrm{os}}_k, 
r^{\mathrm{os}}_k 
\right]^{\mathsf T}, 
\label{eq:os_observation}
\end{equation}
and $\mathbf{o}^{\mathrm{ts}}_{i,k}$ contains the measured position and velocity of the $i$-th target ship in the global horizontal frame:
\begin{equation}
\mathbf{o}^{\mathrm{ts}}_{i,k} = 
\left[ 
\tilde{p}^{\mathrm{ts}}_{x,i,k}, 
\tilde{p}^{\mathrm{ts}}_{y,i,k}, 
\tilde{v}^{\mathrm{ts}}_{x,i,k}, 
\tilde{v}^{\mathrm{ts}}_{y,i,k} 
\right]^{\mathsf T}, 
\quad i=1,\ldots,N. 
\label{eq:ts_observation}
\end{equation}

The observation likelihood function
$\mathcal{O}(\mathbf{o}_{k+1}\mid\mathbf{s}_{k+1},\mathbf{a}_k)$
specifies how the latent vessel states generate sensor measurements.
Consistent with the observation definition in
Eqs.~\eqref{eq:os_observation}--\eqref{eq:ts_observation}, the state of the own ship
is used as direct self-state feedback, whereas the dominant
perception uncertainty is associated with target ship tracking. For each
target ship, the measurement model is written as
\begin{equation}
\mathbf{o}^{\mathrm{ts}}_{i,k}
=
\mathbf{x}^{\mathrm{ts}}_{i,k}
+
\boldsymbol{\epsilon}^{\mathrm{obs}}_{i,k},
\qquad
\boldsymbol{\epsilon}^{\mathrm{obs}}_{i,k}
\sim
\mathcal{N}
\left(
\mathbf{0},
\boldsymbol{\Sigma}^{\mathrm{obs}}_{i,k}
\right),
\label{eq:ts_measurement_model}
\end{equation}
where $\boldsymbol{\epsilon}^{\mathrm{obs}}_{i,k}$ denotes the measurement error of the target ship
and $\boldsymbol{\Sigma}^{\mathrm{obs}}_{i,k}$ is the
corresponding covariance matrix. This Gaussian measurement model specifies
the component of the observation likelihood for the target ship in the POMDP
formulation.

Finally, the reward function $\mathcal{R}(\mathbf{s}_k,\mathbf{a}_k)$
encodes the navigation objective and collision-avoidance requirement. The
objective is to learn a policy that
maximizes the expected discounted return:
\begin{equation}
    J(\pi) =
    \mathbb{E}_{\pi}
    \left[
    \sum_{k=0}^{\infty}
    \gamma^k
    \mathcal{R}(\mathbf{s}_k,\mathbf{a}_k)
    \right],
    \label{eq:objective}
\end{equation}
where $\gamma\in(0,1)$ is the discount factor, and the expectation is
taken over trajectories induced by the policy and the stochastic
observation process.

Since the latent multi-vessel state $\mathbf{s}_k$ is not directly
available, decision making is conditioned on the belief state
\begin{equation}
    \mathbf{b}_k(\mathbf{s}_k)
    \triangleq
    p(\mathbf{s}_k\mid h_k),
    \qquad
    h_k=(\mathbf{o}_{0:k},\mathbf{a}_{0:k-1}),
    \label{eq:belief_definition}
\end{equation}
which summarizes the information contained in the history of noisy
measurements and executed actions. Accordingly, the navigation policy is
written as
\begin{equation}
    \mathbf{a}_k \sim \pi(\cdot\mid \mathbf{b}_k).
    \label{eq:belief_policy}
\end{equation}

\section{Credibility-Aware Learning under Perceptual Uncertainty}
\label{section:credibility_learning}
To address the POMDP formulated in
Section~\ref{section:Problem_Statement}, we develop a credibility-aware
learning and control framework for navigation under perception uncertainty.
A Kalman filter constructs a Gaussian belief state from noisy measurements of target ships, while estimator consistency characterizes the credibility of
that belief. The resulting credibility information is used throughout the
framework: the trust factor modulates value learning in CWVL-PPO, and the
covariance envelope informs the barrier constraints used by CoReCBF-QP. The
learned policy, COLREGs-aware reference, and QP action correction form the
closed-loop navigation architecture shown in Fig.~\ref{fig:frame}.
\begin{figure*}[!t]
    \centering
    \includegraphics[width=\textwidth]{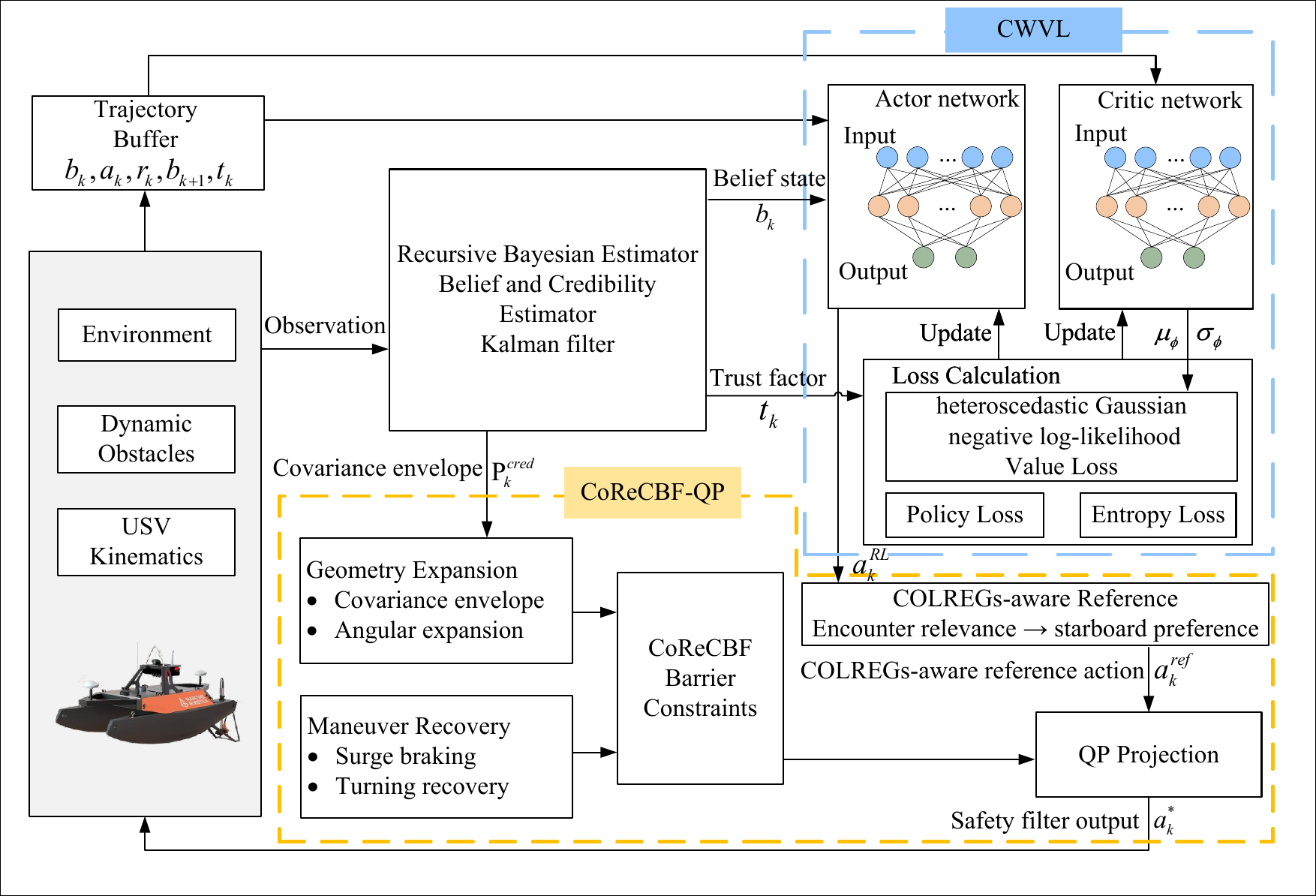}
    \caption{Overview of the proposed credibility-aware framework for USV
    navigation under perception uncertainty.}
    \label{fig:frame}
\end{figure*}
\subsection{Belief-State Estimation and Credibility}
\label{subsec:belief}

To infer the latent state from noisy sensor measurements, we employ a
Kalman-filter-based recursive estimator using the state and observation
models defined in Section~\ref{section:Problem_Statement}. In belief-space
planning, the belief state is generally represented as a posterior
distribution over the latent state conditioned on the history of
observations and actions~\cite{kaelbling1998planning,thrun2005probabilistic}.
However, maintaining this full distribution is computationally
prohibitive for online collision avoidance, especially when multiple
target ships are considered simultaneously. We therefore adopt a Gaussian
approximation, under which the belief state at time step $k$ is
represented by a Gaussian posterior:
\begin{equation}
    \mathbf{b}_k(\mathbf{s}_k) \approx \mathcal{N}(\mathbf{s}_k; \mathbf{\hat{s}}_{k|k}, \mathbf{P}^f_{k|k}),
    \label{eq:belief_approx}
\end{equation}
where $\mathbf{s}_k$ is the latent state defined in~\eqref{eq:state_definition}, and $\mathcal{N}(\mathbf{s}_k; \mathbf{\hat{s}}_{k|k}, \mathbf{P}^f_{k|k})$ denotes the multivariate Gaussian probability density function with mean $\mathbf{\hat{s}}_{k|k}$ and covariance $\mathbf{P}^f_{k|k}$. Here, $\mathbf{\hat{s}}_{k|k}$ is the posterior mean estimate of the latent state obtained from the Kalman-filter update, and $\mathbf{P}^f_{k|k}$ is the filter calculated mean square error (FMSE) matrix.

In the considered multi-target collision-avoidance scenario, the belief state is defined over the joint latent state $\mathbf{s}_k$, which consists of the own ship state and the states of $N$ target ships. The target ships are tracked independently, so their covariance blocks in $\mathbf{P}^{f}_{k|k}$ are assumed to be block diagonal without cross-covariance terms. For target ship $i$, we denote the posterior mean of its substate and the corresponding FMSE matrix by $\hat{\mathbf{x}}^{\mathrm{ts}}_{i,k|k}$ and $\mathbf{P}^f_{i,k|k}$, respectively.
If the assumed process and measurement noise statistics match those of the physical system, $\mathbf{P}^f_{i,k|k}$ is expected to be consistent with the true mean-square-error (TMSE) matrix $\mathbf{P}^m_{i,k|k}$. In practice, however, noise-statistics mismatch and unmodeled dynamics may make the FMSE matrix underestimate the actual error or become inconsistent. Moreover, the exact TMSE matrix $\mathbf{P}^m_{i,k|k}$ is not directly available during deployment. To obtain a practical reliability measure for downstream learning, we adopt the Trust Factor Estimation with Directly Numerical Solving (TFE-DNS) method~\cite{ge2023novel}. Its rationale follows from the innovation-consistency property of the Kalman filter. With correctly specified process and measurement noise statistics, the innovation sequence is approximately zero-mean, and its empirical covariance agrees with the filter-predicted covariance. Noise-statistics mismatch or unmodeled dynamics breaks this consistency, allowing TFE-DNS to use the resulting innovation-covariance deviation to numerically obtain the online TMSE estimate $\mathbf{\hat{P}}^m_{i,k|k}$. The trust factor associated with the $i$-th target ship is then defined as
\begin{equation}
t_{i,k} = \frac{1}{\| \mathbf{P}_{i,k|k}^{f} - \mathbf{\hat{P}}_{i,k|k}^{m} \|_2 + 1}, \quad t_{i,k} \in (0, 1],
\label{eq:credibility_factor}
\end{equation}
where $\|\cdot\|_2$ denotes the matrix 2-norm, evaluated after applying the same normalization to position and velocity components in both the FMSE matrix $\mathbf{P}_{i,k|k}^{f}$ and the online TMSE estimate $\mathbf{\hat{P}}_{i,k|k}^{m}$. A smaller normalized discrepancy indicates better filter consistency and yields a trust factor closer to one.

{
Let $\mathcal I_k$ denote the set of target ships identified as risk active
according to their current distance and TCPA.
The global trust factor is defined as
\begin{equation}
t_k=
\begin{cases}
\displaystyle\min_{i\in\mathcal I_k}t_{i,k},
& \mathcal I_k\neq\varnothing,\\
1, & \text{otherwise}.
\end{cases}
\label{eq:credibility_global}
\end{equation}
}

\subsection{Credibility-Weighted PPO}
\label{subsec:cw_ppo}

The value-learning component of the proposed framework is referred to as
Credibility-Weighted Value Learning (CWVL). It is implemented within PPO
by replacing the standard value loss with a credibility-weighted
heteroscedastic critic loss. 
PPO algorithms typically employ a homoscedastic mean square error (MSE) loss for value estimation~\cite{Schulman2017PPO}, implicitly assuming that all state inputs are equally reliable. However, in our POMDP setting, the belief state is represented by Kalman-filter outputs whose credibility varies with noise-statistics mismatch and unmodeled dynamics. Regressing the critic on targets generated from unreliable beliefs may therefore bias value learning. To reduce this effect, we propose a Credibility-Aware Critic that combines heteroscedastic value~\cite{kendall2017what} modeling with credibility-based reweighting, using the trust factor $t_k$ to attenuate the contribution of unreliable samples.

The critic network is designed to approximate the belief-value 
function, which represents the expected cumulative discounted reward 
from the current belief state:
\begin{equation}
    V^\pi(\mathbf{b}_k) = \mathbb{E}_{\pi}\!\left[\mathcal{R}_k \mid \mathbf{b}_k\right],
    \label{eq:value_function_def}
\end{equation}
where $\mathbb{E}_{\pi}[\cdot]$ denotes the conditional expectation 
when the agent starts from the current belief $\mathbf{b}_k$ and follows policy 
$\pi$ thereafter. The expectation is taken with respect to the 
randomness of the environment dynamics and observations, which jointly 
determine the future rewards. The discounted return $\mathcal{R}_k$ is 
defined as
\begin{equation}
    \mathcal{R}_k \triangleq \sum_{n=0}^{\infty}\gamma^n r_{k+n},
    \label{eq:subsec_return_def}
\end{equation}
where $\gamma \in (0, 1)$ is the discount factor and $r_{k+n}$ is the 
immediate reward at step $k+n$.
Since the infinite-horizon return $\mathcal{R}_k$ is not directly
observable during training, we use a finite-horizon bootstrapped target
$\hat{\mathcal{R}}_k$ as the critic regression target. Computed from
sampled trajectories and bootstrapped value estimates,
$\hat{\mathcal{R}}_k$ provides a noisy estimate of
$V^{\pi}(\mathbf{b}_k)$ and is modeled as an observation corrupted by
belief-dependent target noise:
\begin{equation}
\label{eq:subsec_additive_noise}
\begin{split}
\hat{\mathcal{R}}_k &= V^{\pi}(\mathbf{b}_k)+\varepsilon_k, \\
\mathbb{E}[\varepsilon_k\mid \mathbf{b}_k] &= 0, \quad 
\mathrm{Var}(\varepsilon_k\mid \mathbf{b}_k) = \sigma^2(\mathbf{b}_k),
\end{split}
\end{equation}
where $\varepsilon_k$ is the residual noise with zero conditional mean, 
and $\sigma^2(\mathbf{b}_k)$ denotes the aleatoric variance arising from 
environment stochasticity and partial observability. Taking the 
conditional expectation of~\eqref{eq:subsec_additive_noise} on both 
sides yields $\mathbb{E}[\hat{\mathcal{R}}_k \mid \mathbf{b}_k] = V^{\pi}(\mathbf{b}_k)$, 
so the conditional mean of the bootstrapped target coincides with the 
true belief value. 

To capture this heteroscedastic target uncertainty, whose variance depends on the belief state, the critic is parameterized as a probabilistic network.
Specifically, the critic with parameters $\phi$ outputs a Gaussian predictive density:
\begin{equation}
    p(\hat{\mathcal{R}}_k \mid \mathbf{b}_k; \phi) = 
    \mathcal{N}\!\big(\hat{\mathcal{R}}_k;\, 
    \mu_\phi(\mathbf{b}_k),\, \sigma_\phi^2(\mathbf{b}_k)\big),
    \label{eq:gaussian_pdf}
\end{equation}
where $p(\hat{\mathcal{R}}_k \mid \mathbf{b}_k; \phi)$ is the critic predictive density and $\mathcal{N}(\cdot)$ denotes a univariate normal density. The mean head $\mu_\phi(\mathbf{b}_k)$ estimates the belief value $V^\pi(\mathbf{b}_k)$, whereas the variance head $\sigma_\phi^2(\mathbf{b}_k)$ estimates the belief-dependent variance of the bootstrapped target.

{
Maximizing the log-likelihood of~\eqref{eq:gaussian_pdf} yields the standard heteroscedastic negative log-likelihood (NLL) loss~\cite{kendall2017what}, in which learned predictive variance controls the weight assigned to each squared residual. Studies in deep RL and neural heteroscedastic regression have identified two related limitations of this mechanism: inaccurate variance estimates can distort sample weighting~\cite{Mai2022IVRL}, and variance-dependent gradient scaling can compromise mean fitting~\cite{Immer2023BayesianHeteroscedastic}. The resulting loss captures aleatoric target noise but treats all belief samples as equally credible.
}

This equal-credibility assumption is inadequate in our POMDP setting. The
belief state $\mathbf{b}_k$ is represented by Kalman-filter outputs whose assumed
error statistics may deviate from the physical system, so different
samples carry different levels of filter credibility. We therefore
formulate critic training as a credibility-weighted empirical loss, where
each per-sample loss is multiplied by a non-negative weight reflecting
the reliability of the corresponding belief. We use the global trust
factor $t_k\in(0,1]$ defined in Eq.~\eqref{eq:credibility_global} as this
weight. It is bounded, determined by the discrepancy between FMSE and TMSE, and
approaches one when the FMSE matrix is consistent with the online TMSE
estimate. For numerical stability during optimization, we parameterize
the variance head in log-space as
$s_\phi(\mathbf{b}_k)\triangleq\log\sigma_\phi^2(\mathbf{b}_k)$. Multiplying the
per-sample heteroscedastic NLL by the credibility weight $t_k$ yields
the credibility-weighted value loss at time step $k$:
\begin{equation}
\begin{split}
\mathcal{L}^{V}_k(\phi) = & \frac{1}{2} t_k \Big[ s_\phi(\mathbf{b}_k) \\
& + \exp\!\big(-s_\phi(\mathbf{b}_k)\big)
\big(\hat{\mathcal{R}}_k-\mu_\phi(\mathbf{b}_k)\big)^2 \Big].
\end{split}
\label{eq:subsec_cred_aware_loss}
\end{equation}
Calculating the gradient with respect to the mean prediction explicitly reveals the loss~mechanism:
\begin{equation}
\frac{\partial \mathcal{L}^{V}_k}{\partial \mu_\phi} = 
-\,t_k\,\exp\!\big(-s_\phi(\mathbf{b}_k)\big)\,
\big(\hat{\mathcal{R}}_k-\mu_\phi(\mathbf{b}_k)\big).
\label{eq:subsec_grad_attenuation}
\end{equation}
This formulation reveals two complementary attenuation mechanisms acting on the same update. The per-sample NLL weighting by $t_k$ attenuates gradients for both the mean and variance heads, reducing the influence of low-credibility beliefs on value prediction and uncertainty estimation. Specifically, $\exp(-s_\phi(\mathbf{b}_k))$ provides intrinsic attenuation along the learned aleatoric axis, whereas $t_k$ provides extrinsic attenuation along the filter-credibility axis. Samples from inconsistent filters are assigned lower weights during parameter updates, preventing the critic from overfitting to spurious targets caused by mismatch between the filter and the system. When $t_k \to 1$, the loss recovers standard heteroscedastic regression.

The actor network, parameterized by $\theta$, is optimized via the 
PPO-Clip objective:
\begin{equation}
\label{eq:subsec_ppo_clip}
\mathcal{L}^{\pi}_k(\theta) = \min\!\Big( 
\rho_k(\theta)\,\hat{\mathcal{A}}_k,\; 
\mathrm{clip}\big(\rho_k(\theta),1-\epsilon,1+\epsilon\big)\,\hat{\mathcal{A}}_k 
\Big),
\end{equation}
where $\rho_k(\theta) = \frac{\pi_\theta(\mathbf{a}_k\mid \mathbf{b}_k)}{\pi_{\theta_{\mathrm{old}}}(\mathbf{a}_k\mid \mathbf{b}_k)}$ 
is the probability ratio between the new and old policies, $\epsilon$ is 
the clipping coefficient that bounds the policy update step, and $\hat{\mathcal{A}}_k$ is the advantage estimate used in PPO. The bootstrapped target $\hat{\mathcal{R}}_k$ is used for the critic update in~\eqref{eq:subsec_cred_aware_loss}.

\section{Shielded and Rule-Aware Decision Making under Perceptual Uncertainty}
\label{section:shielded_decision}
{
Credibility-weighted learning reduces the influence of unreliable beliefs on
value estimation during training. At execution time, CoReCBF-QP corrects
policy actions using covariance- and recovery-aware safety constraints
together with a continuous COLREGs-aware reference. CoReCBF constructs the
safety constraint for each target ship from VO geometry expanded according to uncertainty
and terms for braking and turning recovery. The resulting per-target
constraints are incorporated into the QP, whose objective uses a bounded yaw
reference to promote starboard-side avoidance for head-on encounters under
COLREGs Rule~14 and give-way crossing encounters under Rule~15.
}

\subsection{{Covariance- and Recovery-Aware Control Barrier Function}}
\label{subsec:corecbf}

{
CoReCBF uses the posterior state estimates and covariance information from
Section~\ref{subsec:belief} to couple uncertainty-aware collision geometry
with maneuvering recovery. For each target ship
$i\in\{1,\ldots,N\}$, all kinematic quantities below are expressed in the
global horizontal frame. Consistent with Eq.~\eqref{eq:os_kinematics}, let
$\mathbf p_k^{\mathrm{os}}
=[p_{x,k}^{\mathrm{os}},p_{y,k}^{\mathrm{os}}]^{\mathsf T}$ and
$\mathbf v_k^{\mathrm{os}}$ denote the position and translational velocity of
the own ship. Let $\hat{\mathbf p}^{\mathrm{ts}}_{i,k}$ and
$\hat{\mathbf v}^{\mathrm{ts}}_{i,k}$ denote the posterior mean position and
velocity of target ship $i$. The relative position and velocity are defined as
}
\begin{equation}
{
\mathbf r_{i,k}
=\hat{\mathbf p}^{\mathrm{ts}}_{i,k}-\mathbf p_k^{\mathrm{os}},
\qquad
\mathbf v^{\mathrm{rel}}_{i,k}
=\mathbf v_k^{\mathrm{os}}-\hat{\mathbf v}^{\mathrm{ts}}_{i,k}.
}
\label{eq:corecbf_rel_state}
\end{equation}
{
Let $d_{i,k}=\|\mathbf r_{i,k}\|$ denote the separation distance. The
corresponding radial and transverse unit vectors are
$\hat{\mathbf r}_{i,k}=\mathbf r_{i,k}/d_{i,k}$ and
$\hat{\mathbf r}_{\perp,i,k}=
[-(\hat{\mathbf r}_{i,k})_y,(\hat{\mathbf r}_{i,k})_x]^{\mathsf T}$.
With the convention in~\eqref{eq:corecbf_rel_state}, a positive projection of
$\mathbf v^{\mathrm{rel}}_{i,k}$ onto $\hat{\mathbf r}_{i,k}$ indicates
closing motion.
}

{
The credibility analysis in Section~\ref{subsec:belief} characterizes belief
reliability through the discrepancy between the filter-reported FMSE and its
online TMSE estimate. CoReCBF transfers this information to the safety
geometry by constructing a covariance envelope that dominates both matrices:
}
\begin{equation}
{
\begin{aligned}
\mathbf P^{\mathrm{cred}}_{i,k|k}
&=\mathbf P^f_{i,k|k}
+\left[\hat{\mathbf P}^m_{i,k|k}
-\mathbf P^f_{i,k|k}\right]_+,\\
\mathbf\Sigma^{r,\mathrm{cred}}_{i,k}
&=\mathbf E_p\mathbf P^{\mathrm{cred}}_{i,k|k}\mathbf E_p^{\mathsf T}.
\end{aligned}
}
\label{eq:corecbf_credible_cov}
\end{equation}
{
where $[\cdot]_+$ denotes the positive-semidefinite part of a symmetric
matrix, and $\mathbf E_p$ selects the position components. With the localization error of the own ship
neglected, $\mathbf\Sigma^{r,\mathrm{cred}}_{i,k}$ is the
relative-position covariance. This construction satisfies
$\mathbf P^{\mathrm{cred}}_{i,k|k}\succeq\mathbf P^f_{i,k|k}$ and
$\mathbf P^{\mathrm{cred}}_{i,k|k}\succeq\hat{\mathbf P}^m_{i,k|k}$.
}

{
For confidence level $1-\varepsilon_c$, let
$\zeta=\sqrt{\chi^2_{2,1-\varepsilon_c}}$. The credible relative-position
error set is
}
\begin{equation}
{
\mathcal E^{r,\mathrm{cred}}_{i,k}
=\left\{\mathbf e\in\mathbb R^2:
\mathbf e^{\mathsf T}
(\mathbf\Sigma^{r,\mathrm{cred}}_{i,k})^{-1}\mathbf e
\le\zeta^2\right\}.
}
\label{eq:corecbf_confidence_set}
\end{equation}
{
Maximizing the directional projection over the ellipsoidal confidence set
gives
}
\begin{equation}
{
\begin{aligned}
e^{\max}_{\parallel,i,k}
&=\zeta\sqrt{
\hat{\mathbf r}_{i,k}^{\mathsf T}
\mathbf\Sigma^{r,\mathrm{cred}}_{i,k}
\hat{\mathbf r}_{i,k}},\\
e^{\max}_{\perp,i,k}
&=\zeta\sqrt{
\hat{\mathbf r}_{\perp,i,k}^{\mathsf T}
\mathbf\Sigma^{r,\mathrm{cred}}_{i,k}
\hat{\mathbf r}_{\perp,i,k}},\\
d^{\min}_{i,k}
&=d_{i,k}-e^{\max}_{\parallel,i,k},\\
\alpha_{i,k}
&=\arctan\!\left(
\frac{e^{\max}_{\perp,i,k}}{d^{\min}_{i,k}}
\right)
+\arcsin\!\left(
\frac{R_i}{d^{\min}_{i,k}}
\right).
\end{aligned}
}
\label{eq:corecbf_credible_geometry}
\end{equation}
{
The construction applies when $d^{\min}_{i,k}>R_i$ and
$0<\alpha_{i,k}<\pi/2$. At decision step $k$, define
$\chi_{i,k}=(d_{i,k}^2-R_i^2)/R_i^2$ and map the expanded half-angle to the
geometry scale
}
\begin{equation}
{
\lambda_{i,k}
=\frac{\cot^2\alpha_{i,k}}{\chi_{i,k}}.
}
\label{eq:corecbf_geometry_scale}
\end{equation}
{
The credible half-angle is no smaller than the deterministic tangent
half-angle, so $0<\lambda_{i,k}\le1$. In the zero-uncertainty limit,
$\lambda_{i,k}=1$. Moreover,
$\lambda_{i,k}\chi_{i,k}=\cot^2\alpha_{i,k}$, which makes the asymptotes
introduced below coincide with the VO boundary expanded according to uncertainty.
}

{
Under the constant-relative-velocity assumption, the VO method assigns every
velocity inside the collision cone to the forbidden set because its projected
trajectory intersects the protected region. Covariance inflation accounts for
position uncertainty by widening this cone, but retains the same binary
classification over an unbounded prediction horizon. A velocity may therefore
be rejected even when the current separation is large or the own ship can
brake or turn out of the cone before exhausting the safety margin. CoReCBF
retains the covariance-expanded VO boundary as its asymptotic geometry and
introduces distance-dependent surge and turning recovery to distinguish such
recoverable velocities.
}

{
At decision step $k$, the geometry scale $\lambda_{i,k}$ is evaluated from
Eq.~\eqref{eq:corecbf_geometry_scale} and held fixed until the next decision
step. Let $R_i$ denote the combined safety radius. The clearance and the LOS
components of relative velocity are
}
\begin{equation}
{
\begin{aligned}
\Delta_i&=d_i-R_i,
&
\dot{\Delta}_i&=-v_{\parallel,i},\\
v_{\parallel,i}
&=\hat{\mathbf r}_i^{\mathsf T}\mathbf v_i^{\mathrm{rel}},
&
v_{\perp,i}
&=\hat{\mathbf r}_{\perp,i}^{\mathsf T}\mathbf v_i^{\mathrm{rel}}.
\end{aligned}
}
\label{eq:corecbf_clearance_dynamics}
\end{equation}

{
Let $\mathbf E=[\mathbf I_2\ \mathbf 0]$ select the horizontal translation,
and let $\mathbf e_1=[1,0,0]^{\mathsf T}$. The unit inertial-plane
acceleration direction generated by the surge input and its LOS projection
are
}
\begin{equation}
{
\mathbf h_u=
\frac{\mathbf E\mathbf R(\psi^{\mathrm{os}})
\mathbf M^{-1}\mathbf e_1}
{\|\mathbf E\mathbf R(\psi^{\mathrm{os}})
\mathbf M^{-1}\mathbf e_1\|},
\qquad
q_i=\hat{\mathbf r}_i^{\mathsf T}\mathbf h_u.
}
\label{eq:corecbf_surge_authority}
\end{equation}
{
Let $a_u>0$ be the available surge-deceleration scale. The stopping-distance
relation $v^2=2as$ gives the recovery scale $2a_u\Delta_i$. Projecting the
available acceleration magnitude onto the LOS gives
$2a_u\Delta_i|q_i|$. CoReCBF uses the differentiable weight $q_i^2$, which
satisfies $q_i^2\le|q_i|$, and obtains the surge-recovery budget
$2a_u\Delta_iq_i^2$.
}

{
Let $\alpha_r>0$ be the yaw-acceleration scale and
$\sigma_i\in\{-1,1\}$ the selected turning branch. Under nominal constant
angular acceleration, the time required to rotate the heading by $\pi/2$
satisfies
}
\begin{equation}
{
\begin{aligned}
\sigma_i r^{\mathrm{os}}T_{\sigma_i}
+\frac{1}{2}\alpha_rT_{\sigma_i}^2
&=\frac{\pi}{2},\\
T_{\sigma_i}
&=\frac{\sqrt{(r^{\mathrm{os}})^2+\pi\alpha_r}
-\sigma_i r^{\mathrm{os}}}{\alpha_r}>0.
\end{aligned}
}
\label{eq:corecbf_turn_time}
\end{equation}
{
A quarter turn maps a radial closing direction to a tangential direction, so
$T_{\sigma_i}$ defines the turning-response time scale. For
$\Delta_i>0$ and $v_{\parallel,i}>0$, the closing-time condition
}
\begin{equation}
{
\frac{\Delta_i}{v_{\parallel,i}}\ge T_{\sigma_i}
\quad\Longleftrightarrow\quad
v_{\parallel,i}^2
\le\frac{\Delta_i^2}{T_{\sigma_i}^2}
}
\label{eq:corecbf_turn_recovery}
\end{equation}
{
gives the turning-recovery budget. Its signed extension
$\Delta_i|\Delta_i|/T_{\sigma_i}^2$ changes sign across $d_i=R_i$ while
retaining a continuous first derivative.
}

{
Accordingly, the current geometry quantity
$\chi_i=(d_i^2-R_i^2)/R_i^2$ varies with the separation $d_i$, while
$\lambda_{i,k}$ remains fixed over the decision interval. The CoReCBF for
target ship $i$ is
}
\begin{equation}
{
\begin{aligned}
H_{i,k}={}&
\lambda_{i,k}\chi_i v_{\perp,i}^2
-[v_{\parallel,i}]_+^2\\
&+2a_u\Delta_i q_i^2
+\frac{\Delta_i|\Delta_i|}{T_{\sigma_i}^2}.
\end{aligned}
}
\label{eq:corecbf_barrier}
\end{equation}
{
The first two terms retain the expanded VO direction and closing motion. The
last two terms account for surge and turning recovery. Every term has units
of squared velocity. The functions $[z]_+^2$ and $z|z|$ have continuous
first derivatives. Therefore, $H_{i,k}$ is differentiable, and the Lie
derivatives used below are well defined for $d_i>0$.
}

{
\begin{proposition}
\label{prop:corecbf_distance_safety}
For $d_i>0$ and $\lambda_{i,k}>0$, the CoReCBF satisfies
\begin{equation}
H_{i,k}\ge0\Longrightarrow d_i\ge R_i.
\label{eq:corecbf_safety_implication}
\end{equation}
\end{proposition}
\begin{IEEEproof}
If $0<d_i<R_i$, then $\chi_i<0$ and $\Delta_i<0$. The first three terms
of~\eqref{eq:corecbf_barrier} are nonpositive, while the signed turning term
is negative. Hence, $H_{i,k}<0$, which proves the implication by
contraposition. At $d_i=R_i$,
$H_{i,k}=-[v_{\parallel,i}]_+^2$, so $H_{i,k}\ge0$ requires
$v_{\parallel,i}\le0$.
\end{IEEEproof}
}

{
For the velocity space interpretation at decision step $k$, fix the geometry
at its sampled value, so that $d_i=d_{i,k}>R_i$ and
$\chi_i=\chi_{i,k}$, and define
}
\begin{equation}
{
B_{i,k}
=2a_u\Delta_iq_i^2
+\frac{\Delta_i^2}{T_{\sigma_i}^2}>0.
}
\label{eq:corecbf_recovery_budget}
\end{equation}
{
In the approaching half-plane $v_{\parallel,i}>0$, the boundary
$H_{i,k}=0$ becomes
}
\begin{equation}
{
v_{\parallel,i}^2
-\lambda_{i,k}\chi_{i,k}v_{\perp,i}^2
=B_{i,k}.
}
\label{eq:corecbf_hyperbola}
\end{equation}
{
This is a hyperbola with asymptotes
}
\begin{equation}
{
v_{\parallel,i}
=\cot\alpha_{i,k}|v_{\perp,i}|,
\qquad
\frac{|v_{\perp,i}|}{v_{\parallel,i}}
=\tan\alpha_{i,k}.
}
\label{eq:corecbf_asymptotes}
\end{equation}
{
Thus, its asymptotes reproduce the boundary of the covariance-expanded VO
cone. Define the corresponding VO and CoReCBF forbidden sets as
}
\begin{equation}
{
\begin{aligned}
\mathcal V^{\mathrm{VO}}_{i,k}(\alpha_{i,k})
&=\left\{(v_{\parallel,i},v_{\perp,i}):
v_{\parallel,i}>0,\ 
v_{\parallel,i}^2
\ge\lambda_{i,k}\chi_{i,k}v_{\perp,i}^2\right\},\\
\mathcal V^-_{H,i,k}
&=\left\{(v_{\parallel,i},v_{\perp,i}):H_{i,k}<0\right\}.
\end{aligned}
}
\label{eq:corecbf_forbidden_sets}
\end{equation}
{
\begin{proposition}
For $d_i>R_i$,
$\mathcal V^-_{H,i,k}\subsetneq
\mathcal V^{\mathrm{VO}}_{i,k}(\alpha_{i,k})$.
\end{proposition}
\begin{IEEEproof}
If $H_{i,k}<0$, then
$v_{\parallel,i}^2>
\lambda_{i,k}\chi_{i,k}v_{\perp,i}^2+B_{i,k}>
\lambda_{i,k}\chi_{i,k}v_{\perp,i}^2$.
The VO boundary satisfies $H_{i,k}=B_{i,k}>0$, so the inclusion is strict.
\end{IEEEproof}
Consequently, CoReCBF does not classify every velocity inside the
covariance-expanded VO cone as forbidden. A cone-interior velocity remains
admissible when the available clearance and maneuvering authority are
sufficient to prevent $d_i$ from reaching $R_i$ through surge braking,
turning, or their combined action.
Figure~\ref{fig:corecbf_geometry_velocity} summarizes the credible angular
expansion and the corresponding CoReCBF boundary in relative-velocity space.
}

{
To expose both control channels, let
$\mathbf e_3=[0,0,1]^{\mathsf T}$,
$\mathbf B_a=[\mathbf e_1\ \mathbf e_3]$, and
}
\begin{equation}
{
\begin{aligned}
\mathbf f_\nu
&=-\mathbf M^{-1}
\left[\mathbf C(\boldsymbol\nu^{\mathrm{os}})
\boldsymbol\nu^{\mathrm{os}}
+\mathbf D\boldsymbol\nu^{\mathrm{os}}\right],\\
\dot{\boldsymbol\nu}^{\mathrm{os}}
&=\mathbf f_\nu+\mathbf M^{-1}\mathbf B_a\mathbf a.
\end{aligned}
}
\label{eq:corecbf_control_affine}
\end{equation}
{
For constant velocity of the target ship, with
$\mathbf J=\begin{bmatrix}0&-1\\1&0\end{bmatrix}$,
}
\begin{equation}
{
\begin{aligned}
L_f\mathbf v_i^{\mathrm{rel}}
&=\mathbf E\mathbf R(\psi^{\mathrm{os}})\mathbf f_\nu\\
&\quad+r^{\mathrm{os}}\mathbf J\mathbf E
\mathbf R(\psi^{\mathrm{os}})\boldsymbol\nu^{\mathrm{os}},\\
L_g\mathbf v_i^{\mathrm{rel}}
&=\mathbf E\mathbf R(\psi^{\mathrm{os}})
\mathbf M^{-1}\mathbf B_a,\\
L_fr^{\mathrm{os}}
&=\mathbf e_3^{\mathsf T}\mathbf f_\nu,
&
L_gr^{\mathrm{os}}
&=\mathbf e_3^{\mathsf T}\mathbf M^{-1}\mathbf B_a.
\end{aligned}
}
\label{eq:corecbf_relative_dynamics}
\end{equation}
{
The first and second columns of the input Lie derivatives correspond to the
surge thrust $\tau_u$ and yaw moment $\tau_r$. Holding $\lambda_{i,k}$ and
the selected turning branch fixed between decision steps, define
}
\begin{equation}
{
\begin{aligned}
q_{\perp,i}
&=\hat{\mathbf r}_{\perp,i}^{\mathsf T}\mathbf h_u,\\
\mathbf A_{i,k}
&=2\left(
\lambda_{i,k}\chi_i v_{\perp,i}\hat{\mathbf r}_{\perp,i}
-[v_{\parallel,i}]_+\hat{\mathbf r}_i
\right),\\
c_i
&=\frac{2\sigma_i\Delta_i|\Delta_i|}
{T_{\sigma_i}^2
\sqrt{(r^{\mathrm{os}})^2+\pi\alpha_r}}.
\end{aligned}
}
\label{eq:corecbf_lie_abbreviations}
\end{equation}
{
The Lie derivatives of~\eqref{eq:corecbf_barrier} are
}
\begin{equation}
{
\begin{aligned}
L_fH_{i,k}={}&
\frac{2\left([v_{\parallel,i}]_+
-\lambda_{i,k}v_{\parallel,i}\right)v_{\perp,i}^2}{d_i}
+\mathbf A_{i,k}^{\mathsf T}L_f\mathbf v_i^{\mathrm{rel}}\\
&-2a_uv_{\parallel,i}q_i^2
-4a_u\Delta_iq_iq_{\perp,i}
\left(r^{\mathrm{os}}+\frac{v_{\perp,i}}{d_i}\right)\\
&-\frac{2|\Delta_i|v_{\parallel,i}}{T_{\sigma_i}^2}
+c_iL_fr^{\mathrm{os}},\\
L_gH_{i,k}={}&
\mathbf A_{i,k}^{\mathsf T}L_g\mathbf v_i^{\mathrm{rel}}
+c_iL_gr^{\mathrm{os}}
=\begin{bmatrix}L_{g_u}H_{i,k}&L_{g_r}H_{i,k}\end{bmatrix}.
\end{aligned}
}
\label{eq:corecbf_lie_derivatives}
\end{equation}
{
Hence,
}
\begin{equation}
{
\dot H_{i,k}
=L_fH_{i,k}
+L_{g_u}H_{i,k}\tau_u
+L_{g_r}H_{i,k}\tau_r.
}
\label{eq:corecbf_control_channels}
\end{equation}
{
Both generalized inputs therefore enter one affine safety constraint. With
$k_H>0$, the CoReCBF condition for target ship $i$ is
}
\begin{equation}
{
L_fH_{i,k}
+L_{g_u}H_{i,k}\tau_u
+L_{g_r}H_{i,k}\tau_r
+k_HH_{i,k}\ge0.
}
\label{eq:corecbf_condition}
\end{equation}

{
\begin{assumption}
Fix a target ship $i$, and let
$0=\vartheta_0<\vartheta_1<\cdots$ denote the CoReCBF update instants.
Let $\mathcal X_i$ be an open state domain with $d_i>0$.

\textup{(i)} At every update instant, the covariance-aware geometry satisfies
\[
d^{\min}_{i,k}>R_i,
\qquad
0<\alpha_{i,k}<\frac{\pi}{2}.
\]

\textup{(ii)} The parameters $\lambda_{i,k}$ and $\sigma_i$ remain fixed on
$[\vartheta_k,\vartheta_{k+1})$. For $k\ge1$, the barrier constructed from
the updated geometry satisfies
\[
H_{i,k}(\vartheta_k)\ge0.
\]

\textup{(iii)} For every state in $\mathcal X_i$ satisfying
$H_{i,k}\ge0$, there exists an action $\mathbf a\in\mathcal A$ such that
\[
L_fH_{i,k}+L_gH_{i,k}\mathbf a+k_HH_{i,k}\ge0.
\]

\textup{(iv)} The feedback selected by the CoReCBF-QP for a single target ship is
locally Lipschitz on $\mathcal X_i$.

\textup{(v)} The initial state belongs to $\mathcal X_i$ and satisfies
\[
H_{i,0}(\vartheta_0)\ge0.
\]
\end{assumption}
}

{
\begin{theorem}[Forward invariance and distance safety for one target ship]
Under Assumption~1, the closed-loop trajectory in $\mathcal X_i$ generated
by the CoReCBF-QP for a single target ship satisfies
\begin{equation}
H_{i,k}(t)\ge0,
\qquad
d_i(t)\ge R_i,
\qquad
t\in[\vartheta_k,\vartheta_{k+1}),\quad k\ge0.
\end{equation}
\end{theorem}

\begin{IEEEproof}
On $[\vartheta_k,\vartheta_{k+1})$, Assumption~1(ii) fixes the updated
parameters of $H_{i,k}$. The control action selected by the QP satisfies
\eqref{eq:corecbf_condition}, and hence
\begin{equation}
\dot H_{i,k}(t)+k_HH_{i,k}(t)\ge0.
\end{equation}
It follows that
\begin{equation}
\frac{\mathrm d}{\mathrm dt}
\left[
e^{k_H(t-\vartheta_k)}H_{i,k}(t)
\right]\ge0,
\end{equation}
which gives
\begin{equation}
H_{i,k}(t)
\ge
e^{-k_H(t-\vartheta_k)}H_{i,k}(\vartheta_k)
\ge0.
\end{equation}
Assumption~1(v) establishes the result on the initial update interval, and
Assumption~1(ii) extends it successively to every subsequent interval.
Proposition~\ref{prop:corecbf_distance_safety} then yields
$d_i(t)\ge R_i$.
\end{IEEEproof}
}

\begin{figure*}[t]
  \centering
  \subfloat[Credibility-expanded VO geometry.]{\includegraphics[width=0.4\textwidth]{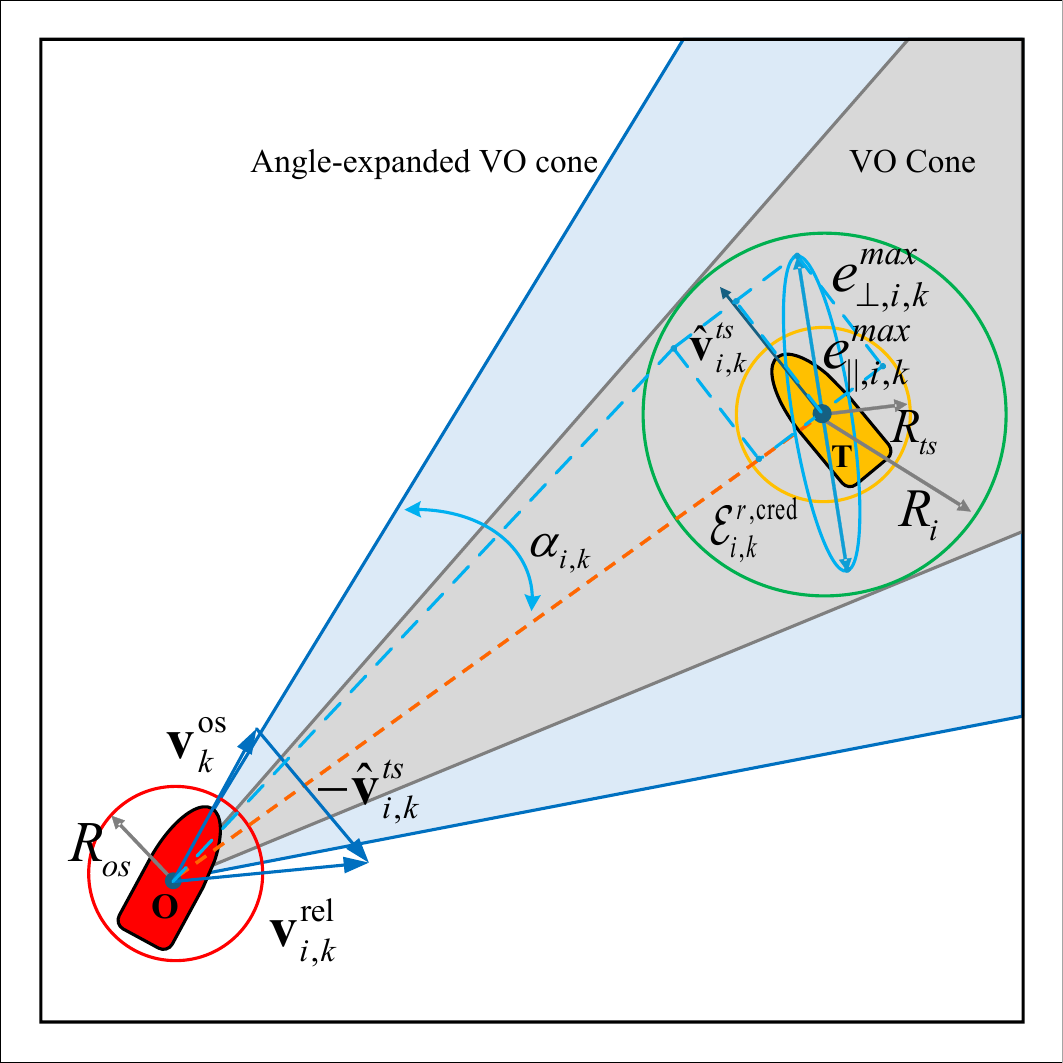}%
  \label{fig:corecbf_geometry}}
  \hfil
  \subfloat[CoReCBF boundary in relative-velocity space.]{\includegraphics[width=0.4\textwidth]{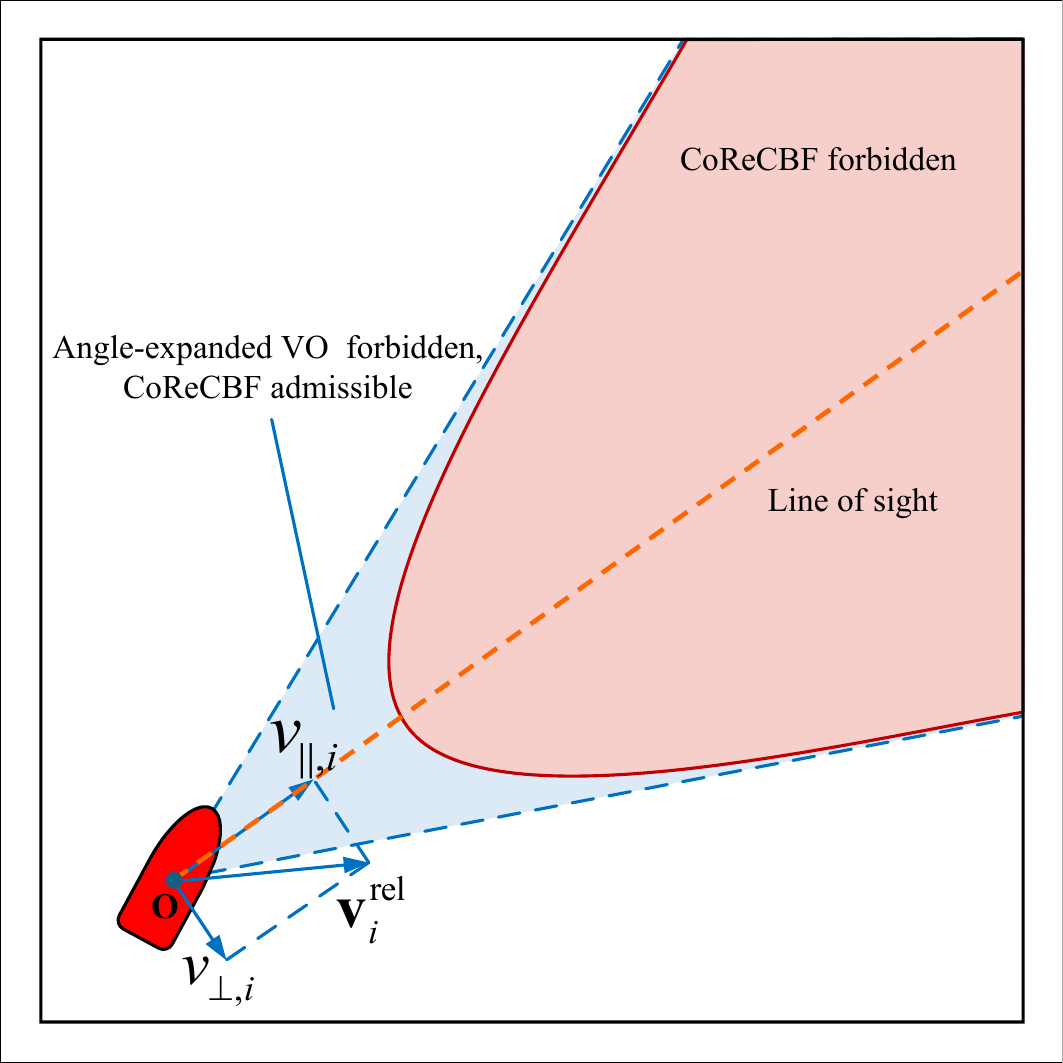}%
  \label{fig:corecbf_velspace}}
  \caption{Credibility-aware collision geometry and its CoReCBF interpretation. (a) The credible relative-position uncertainty expands the nominal VO half-angle to $\alpha_{i,k}$. (b) In relative-velocity coordinates, the blue dashed rays denote the asymptotes of the credibility-expanded VO cone, while the red curve denotes the CoReCBF boundary $H_{i,k}=0$. The light-blue region is forbidden by the expanded VO but remains admissible under CoReCBF, whereas the red region satisfies $H_{i,k}<0$.}
  \label{fig:corecbf_geometry_velocity}
\end{figure*}

\subsection{{CoReCBF-QP for Multiple Target Ships}}
\label{subsec:corecbf_qp}

{
To construct the continuous COLREGs-aware reference, we express the relative
position of each target ship in the encounter frame of the own ship. Its $x$-axis
points forward, and its $y$-axis points to port:
}
\begin{equation}
{
\begin{aligned}
\begin{bmatrix}x_{i,k}\\y_{i,k}\end{bmatrix}
&=
\begin{bmatrix}
\cos\psi_k^{\mathrm{os}}&\sin\psi_k^{\mathrm{os}}\\
-\sin\psi_k^{\mathrm{os}}&\cos\psi_k^{\mathrm{os}}
\end{bmatrix}
\mathbf r_{i,k},\\
\beta_{i,k}
&=\operatorname{atan2}(y_{i,k},x_{i,k})\in(-\pi,\pi].
\end{aligned}
}
\label{eq:colregs_bearing}
\end{equation}
{
Under this convention, $\beta_{i,k}<0$ indicates that the target ship is on
the starboard side, whereas $\beta_{i,k}>0$ indicates the port side. Assuming constant relative velocity and \(|v^{rel}_{i,k}|>0\), the time and distance at the closest point of
approach are
}
\begin{equation}
{
t_{i,k}^{\mathrm{CPA}}
=\frac{\mathbf r_{i,k}^{\mathsf T}
\mathbf v_{i,k}^{\mathrm{rel}}}
{\|\mathbf v_{i,k}^{\mathrm{rel}}\|^2},
\qquad
d_{i,k}^{\mathrm{CPA}}
=\left\|\mathbf r_{i,k}
-\mathbf v_{i,k}^{\mathrm{rel}}t_{i,k}^{\mathrm{CPA}}\right\|.
}
\label{eq:colregs_cpa}
\end{equation}
{
A positive
$t_{i,k}^{\mathrm{CPA}}$ indicates that the closest approach lies in the
future. We use two bearing boundaries to describe the encounter geometry.
The angle $\theta_h$ defines the half-width of the head-on region, whereas
$\theta_{\lim}$ sets the outer bearing limit of the Rule~15 give-way crossing
region, with $0<\theta_h<\theta_{\lim}<\pi$. The corresponding angular
activations are
}
\begin{equation}
{
\begin{aligned}
\Psi_{i,k}^{H}
&=S\!\left[\kappa_\beta
(\cos\beta_{i,k}-\cos\theta_h)\right],\\
\Psi_{i,k}^{SC}
&=S\!\left[\kappa_\beta(-\beta_{i,k}-\theta_h)\right]
S\!\left[\kappa_\beta(\beta_{i,k}+\theta_{\lim})\right],
\end{aligned}
}
\label{eq:colregs_memberships}
\end{equation}
{
where $S(z)=1/(1+\exp(-z))$, and $\kappa_\beta$ controls the angular
transition. The term $\Psi_{i,k}^{H}$ activates around the COLREGs Rule~14
head-on region. The term $\Psi_{i,k}^{SC}$ forms a smooth window over
$-\theta_{\lim}\le\beta_{i,k}\le-\theta_h$, corresponding to Rule~15
crossings in which the target ship lies on the own ship's starboard side.

We combine the angular activations with the predicted distance and timing of
the closest approach. The CPA relevance gate is
}
\begin{equation}
{
\begin{aligned}
G_{i,k}
&=\left[1-\frac{d_{i,k}^{\mathrm{CPA}}}{d_{\mathrm{safe}}}\right]_+
S\!\left(\kappa_t t_{i,k}^{\mathrm{CPA}}\right)\\
&\quad{}
S\!\left[\kappa_t
\left(T_{\mathrm{CPA}}-t_{i,k}^{\mathrm{CPA}}\right)\right],
\qquad \|\mathbf v_{i,k}^{\mathrm{rel}}\|>0.
\end{aligned}
}
\label{eq:colregs_risk_gate}
\end{equation}
{
Here, $d_{\mathrm{safe}}$ specifies the relevant CPA-distance range,
$T_{\mathrm{CPA}}$ defines the prediction horizon, and $\kappa_t$ controls
the temporal transitions at its boundaries. Thus, $G_{i,k}$ assigns greater
weight to encounters with a small predicted passing distance within the
prediction horizon. When $\|\mathbf v_{i,k}^{\mathrm{rel}}\|=0$, the CPA time
is not unique and $G_{i,k}$ is set to zero.
}

{
The reference activation induced by target ship $i$ is then
}
\begin{equation}
{
\Phi_{i,k}
=G_{i,k}
\left[1-(1-\Psi_{i,k}^{H})(1-\Psi_{i,k}^{SC})\right]
\in[0,1].
}
\label{eq:colregs_target_duty}
\end{equation}
{
The union term combines the Rule~14 and Rule~15 angular activations.
Consequently, $\Phi_{i,k}$ represents the strength of the starboard yaw
reference associated with target ship $i$, accounting jointly for encounter
geometry, predicted passing distance, and closest-approach time. The zero-gate
case also yields $\Phi_{i,k}=0$.

For multiple target ships, the per-target activations are combined using a
complement-product aggregation:
}
\begin{equation}
{
\bar\Phi_k
=1-\prod_{i=1}^{N}(1-\Phi_{i,k})
\in[0,1].
}
\label{eq:colregs_aggregate_duty}
\end{equation}
{
The aggregate $\bar\Phi_k$ is zero when no target ship activates the reference
and increases when one or more target ships produce stronger activations. It
therefore provides a bounded encounter-level intensity for constructing the
common starboard yaw reference in the QP for multiple target ships.
}

{
To construct this common yaw reference from the policy output, let
}
\begin{equation}
{
\mathbf a_k^{\mathrm{RL}}
=\begin{bmatrix}
\tau_{u,k}^{\mathrm{RL}}\\
\tau_{r,k}^{\mathrm{RL}}
\end{bmatrix}.
}
\label{eq:colregs_policy_action}
\end{equation}
{
From the 3-DOF control-affine dynamics in
Eq.~\eqref{eq:corecbf_control_affine}, the yaw component is
}
\begin{equation}
{
\begin{aligned}
\dot r^{\mathrm{os}}
&=\mathbf e_3^{\mathsf T}\dot{\boldsymbol\nu}^{\mathrm{os}}\\
&=\mathbf e_3^{\mathsf T}\mathbf f_\nu
+\mathbf e_3^{\mathsf T}\mathbf M^{-1}\mathbf e_1\,\tau_u
+\mathbf e_3^{\mathsf T}\mathbf M^{-1}\mathbf e_3\,\tau_r .
\end{aligned}
}
\label{eq:colregs_yaw_channel}
\end{equation}
{
For a yaw-channel increment at a fixed state and unchanged surge input,
subtracting the nominal relation gives
}
\begin{equation}
{
\Delta\dot r^{\mathrm{os}}
=b_r\Delta\tau_r,
\qquad
b_r:=\mathbf e_3^{\mathsf T}\mathbf M^{-1}\mathbf e_3>0.
}
\label{eq:colregs_yaw_increment}
\end{equation}
{
For the diagonal inertia matrix used in the simulations, $b_r=1/m_{33}$.
Using the yaw-acceleration scale $\alpha_r$ defined by the turning-recovery
model in Eq.~\eqref{eq:corecbf_turn_time}, we define the COLREGs-induced
yaw-acceleration increment as
}
\begin{equation}
{
\Delta\dot r_k^{\mathrm{ref}}
=-\eta_{\mathrm{ref}}\alpha_r\bar\Phi_k,
}
\label{eq:colregs_yaw_accel_reference}
\end{equation}
{
where $\eta_{\mathrm{ref}}\in[0,1]$ controls the reference strength and
$\tau_r<0$ denotes a starboard yaw moment. Substitution into the yaw-channel
relation gives
}
\begin{equation}
{
\begin{aligned}
b_r\Delta\tau_{r,k}^{\mathrm{ref}}
&=-\eta_{\mathrm{ref}}\alpha_r\bar\Phi_k,\\
\Delta\tau_{r,k}^{\mathrm{ref}}
&=-\frac{\eta_{\mathrm{ref}}\alpha_r\bar\Phi_k}{b_r}.
\end{aligned}
}
\label{eq:colregs_yaw_moment_reference}
\end{equation}
{
The resulting COLREGs-aware reference action is
}
\begin{equation}
{
\begin{aligned}
\mathbf a_k^{\mathrm{ref}}
&:=\mathbf a_k^{\mathrm{RL}}
+\begin{bmatrix}
0\\
\Delta\tau_{r,k}^{\mathrm{ref}}
\end{bmatrix}\\
&=\begin{bmatrix}
\tau_{u,k}^{\mathrm{RL}}\\[1mm]
\tau_{r,k}^{\mathrm{RL}}
-\dfrac{\eta_{\mathrm{ref}}\alpha_r\bar\Phi_k}{b_r}
\end{bmatrix}.
\end{aligned}
}
\label{eq:colregs_reference_action}
\end{equation}
{
When $\bar\Phi_k=0$, the reference recovers the policy action. As the
aggregate activation increases, the reference continuously increases the
starboard yaw correction. Using this common COLREGs-aware action as its
reference, the execution QP combines the CoReCBF constraints associated with
the selected target ships:
}

\begin{equation}
{
\begin{aligned}
\mathbf a_k^*
=\arg\min_{\mathbf a\in\mathcal A}\quad&
\frac{1}{2}
(\mathbf a-\mathbf a_k^{\mathrm{ref}})^{\mathsf T}
\mathbf W
(\mathbf a-\mathbf a_k^{\mathrm{ref}})\\
\mathrm{s.t.}\quad&
L_fH_{i,k}+L_gH_{i,k}\mathbf a+k_HH_{i,k}\ge0,\\
&i=1,\ldots,N.
\end{aligned}
}
\label{eq:multiship_corecbf_qp}
\end{equation}
{
where $\mathbf W=\mathbf W^{\mathsf T}\succ0$. The QP tracks the
COLREGs-aware reference over the actions admitted by the
assembled CoReCBF constraints. When the reference is feasible, the solution
satisfies $\mathbf a_k^*=\mathbf a_k^{\mathrm{ref}}$. Otherwise,
$\mathbf a_k^*$ is its closest feasible realization under the metric induced
by $\mathbf W$. When needed in dense encounters involving multiple ships, a shared
nonnegative slack variable is added to the barrier constraints and penalized
in the objective.
}

\subsection{{Reward Design}}
\label{subsec:reward_design}

{
The abstract reward function introduced in Eq.~\eqref{eq:objective} is
instantiated below for the navigation task considered in this study. It is
used during CWVL-PPO training, while the proposed CoReCBF-QP determines the
executable action during the subsequent execution stage. The step reward is
}
\begin{equation}
{
r_k =
r_k^{\mathrm{prog}}
+r_k^{\mathrm{term}}
+r_k^{\mathrm{risk}}
+r_k^{\mathrm{time}}
+r_k^{\mathrm{cte}}
+r_k^{\mathrm{smooth}}.
}
\label{eq:reward_total}
\end{equation}

{
The progress, time, path deviation, and smoothness terms are
}
\begin{equation}
{
\begin{aligned}
r_k^{\mathrm{prog}} &= w_{\mathrm{prog}}(d_{k-1}-d_k),
&r_k^{\mathrm{time}} &= -0.01,\\
r_k^{\mathrm{cte}} &= -w_{\mathrm{cte}}e_{\mathrm{cte},k},
&r_k^{\mathrm{smooth}} &=
-w_{\mathrm{yaw}}\left(\dot{\psi}_k^{\mathrm{os}}\right)^2.
\end{aligned}
}
\label{eq:reward_regular}
\end{equation}
{
Let $\mathbf p_{\mathrm{start}}$ and $\mathbf p_{\mathrm{goal}}$ denote the
start and goal positions in the global horizontal frame. The goal distance
and path deviation are
}
\begin{equation}
{
\begin{aligned}
d_k
&=\left\|\mathbf p_{\mathrm{goal}}-\mathbf p_k^{\mathrm{os}}\right\|_2,\\
e_{\mathrm{cte},k}
&=
\frac{
\left|
(\mathbf p_k^{\mathrm{os}}-\mathbf p_{\mathrm{start}})
\times
(\mathbf p_{\mathrm{goal}}-\mathbf p_{\mathrm{start}})
\right|
}{
\left\|\mathbf p_{\mathrm{goal}}-\mathbf p_{\mathrm{start}}\right\|_2
}.
\end{aligned}
}
\label{eq:reward_geometry}
\end{equation}
{
The parameter values are $w_{\mathrm{prog}}=30.0$,
$w_{\mathrm{cte}}=0.1$, and $w_{\mathrm{yaw}}=0.05$.

Terminal outcomes are represented by
}
\begin{equation}
{
r_k^{\mathrm{term}}
=R_{\mathrm{goal}}\mathbb{I}_{\mathrm{goal}}
+R_{\mathrm{col}}\mathbb{I}_{\mathrm{collision}}
+R_{\mathrm{timeout}}\mathbb{I}_{\mathrm{timeout}}.
}
\label{eq:reward_terminal}
\end{equation}
{
The indicators denote goal arrival, collision, and episode timeout,
respectively, with $R_{\mathrm{goal}}=200.0$,
$R_{\mathrm{col}}=-300.0$, and $R_{\mathrm{timeout}}=-50.0$.

The predicted encounter risk is penalized before a collision occurs. Let
$\mathcal C_k$ denote the set of target ships satisfying
$0<t_{i,k}^{\mathrm{CPA}}<T_{\mathrm{risk}}$ and
$d_{i,k}^{\mathrm{CPA}}<d_{\mathrm{warn}}$:
}
\begin{equation}
{
\begin{aligned}
r_k^{\mathrm{risk}}={}&-w_{\mathrm{risk}}
\sum_{i\in\mathcal C_k}
\left(\frac{d_{\mathrm{warn}}-d_{i,k}^{\mathrm{CPA}}}
{d_{\mathrm{warn}}}\right)^2\\[-1mm]
&\quad\times\left[\frac{1}{2}+\frac{1}{2}
\left(1-\frac{t_{i,k}^{\mathrm{CPA}}}{T_{\mathrm{risk}}}\right)\right].
\end{aligned}
}
\label{eq:reward_risk}
\end{equation}
{
where $d_{i,k}^{\mathrm{CPA}}$ and $t_{i,k}^{\mathrm{CPA}}$ are defined in
Eq.~\eqref{eq:colregs_cpa}. The parameter values are
$w_{\mathrm{risk}}=3.0$,
$d_{\mathrm{warn}}=4.0\,\mathrm{m}$, and
$T_{\mathrm{risk}}=12.0\,\mathrm{s}$. This term therefore increases as a
potential encounter becomes both closer and more imminent.

The resulting step reward is used to compute discounted returns and
advantages during CWVL-PPO training. After training, the policy action is
passed to the proposed CoReCBF-QP, which selects the executable action under
the assembled safety constraints and the continuous COLREGs-aware reference.
The complete training and execution procedure is summarized in
Algorithm~\ref{alg:training}.
}

\begin{algorithm}[!t]
\caption{{CWVL-PPO Training and CoReCBF-QP Execution}}
\label{alg:training}
\small
\begin{algorithmic}[1]
\STATE \textbf{Input:} Kalman filters $\text{FilterInit}(\cdot)$, $\text{Filter}(\cdot)$; horizon $N_{\text{steps}}$, epochs $K$, total iterations $M$
\STATE \textbf{Initialize:} Actor parameters $\theta$, Critic parameters $\phi$, empty buffer $\mathcal{D}$

\STATE {\textbf{Stage I: CWVL-PPO training}}
\FOR{$\text{iter} = 1, \dots, M$}
    \STATE $\mathbf{o}_0 \gets$ Reset env.\ \& get initial obs.
    \STATE $\mathbf{b}_0 \gets \text{FilterInit}(\mathbf{o}_0)$
    \FOR{$k = 0, \dots, N_{\text{steps}}-1$}
        \FOR{$i = 1, \dots, N$}
            \STATE Compute trust factor $t_{i,k}$ via Eq.~\eqref{eq:credibility_factor}
        \ENDFOR
        \STATE {Determine risk-active targets from current distance and TCPA}
        \STATE Aggregate global trust factor $t_k$ via Eq.~\eqref{eq:credibility_global}
        \STATE {Sample action $\mathbf{a}_k \sim \pi_{\theta}(\cdot\mid \mathbf{b}_k)$}
        \STATE Execute $\mathbf{a}_k$, receive reward $r_k$ and observation $\mathbf{o}_{k+1}$
        \STATE $\mathbf{b}_{k+1}\gets \text{Filter}(\mathbf{b}_k,\mathbf{a}_k,\mathbf{o}_{k+1})$
        \STATE {Store $(\mathbf{b}_k,\mathbf{a}_k,r_k,t_k,\mathbf{b}_{k+1})$ in $\mathcal{D}$}
    \ENDFOR
    \STATE Compute advantages $\hat{\mathcal{A}}_k$ and bootstrapped returns $\hat{\mathcal{R}}_k$ via GAE
    \STATE $\theta_{\mathrm{old}}\gets\theta$
    \FOR{$\text{epoch}=1,\dots,K$}
        \STATE Update Critic via credibility-weighted loss $\mathcal{L}^{V}(\phi)$ in Eq.~\eqref{eq:subsec_cred_aware_loss}
        \STATE Update Actor via PPO-Clip loss $\mathcal{L}^{\pi}(\theta)$ in Eq.~\eqref{eq:subsec_ppo_clip} using $\pi_{\theta_{\mathrm{old}}}$
    \ENDFOR
    \STATE Clear buffer $\mathcal{D}$
\ENDFOR
\STATE {\textbf{Training output:} Optimized Actor parameters $\theta$ and Critic parameters $\phi$}

\STATE {\textbf{Stage II: CoReCBF-QP execution}}
\WHILE{{the navigation episode is active}}
\STATE {Given trained policy $\pi_\theta$, obtain $\mathbf{a}^{\mathrm{RL}}_k \sim \pi_\theta(\cdot\mid \mathbf{b}_k)$}
{
\FOR{$i = 1, \dots, N$}
    \STATE Compute $\beta_{i,k}$, $t_{i,k}^{\mathrm{CPA}}$, $d_{i,k}^{\mathrm{CPA}}$, and $\Phi_{i,k}$ via Eqs.~\eqref{eq:colregs_bearing}--\eqref{eq:colregs_target_duty}
\ENDFOR
\STATE Aggregate $\bar\Phi_k$ via Eq.~\eqref{eq:colregs_aggregate_duty}
\STATE Construct $\mathbf a_k^{\mathrm{ref}}$ via Eq.~\eqref{eq:colregs_reference_action}
\STATE Select target ships according to their current distance and $t_{i,k}^{\mathrm{CPA}}$
\FOR{each selected target ship $i$}
    \STATE Construct the relative state and covariance envelope via Eqs.~\eqref{eq:corecbf_rel_state} and~\eqref{eq:corecbf_credible_cov}
    \STATE Construct the covariance-aware collision geometry, $H_{i,k}$, and its Lie derivatives via Eqs.~\eqref{eq:corecbf_credible_geometry}, \eqref{eq:corecbf_geometry_scale}, \eqref{eq:corecbf_barrier}, and~\eqref{eq:corecbf_lie_derivatives}
\ENDFOR
}
\STATE {Solve Eq.~\eqref{eq:multiship_corecbf_qp} for $\mathbf a_k^*$ using the resulting CoReCBF constraints}
\STATE {Execute $\mathbf{a}_k^*$ and obtain $\mathbf o_{k+1}$}
\STATE {$\mathbf b_{k+1}\gets\text{Filter}(\mathbf b_k,\mathbf a_k^*,\mathbf o_{k+1})$}
\ENDWHILE
\end{algorithmic}
\end{algorithm}

\section{Numerical Results}
\label{section:Experiments}

{
To evaluate the proposed framework, a series of experiments are conducted. First, the simulation environment and experimental setup are defined. Specifically, reproducible scenarios with three to six target ships are constructed, and a scheduled period of mismatch in state estimation is introduced. Under these settings, the proposed CWVL is compared with baselines based on learning and optimization to evaluate robustness under perceptual inconsistency. Furthermore, performance across different numbers of target ships is evaluated under three execution configurations of the proposed CWVL policy: without a safety layer, with CBF-VO, and with the proposed CoReCBF-QP. Finally, success, speed, path length, COLREGs event compliance, and mean online control time are reported.
}

\subsection{Simulation Setup}
\label{subsec:setup}
All simulations are implemented in Python/PyTorch on a desktop computer with an Intel i5-12600KF CPU, an NVIDIA RTX 5060 GPU, and 16 GB of memory. The environment is a $32\,\mathrm{m}\times32\,\mathrm{m}$ rectangular area whose boundaries are treated as static obstacles. The modeled own ship and target ships have a hull length $L=2.0\,\mathrm{m}$ and a beam $B=1.08\,\mathrm{m}$, so each side of the domain spans $16L$. Initial and goal states are placed in opposite regions of the map. The proposed CWVL policy is trained in scenarios with three to six target ships. Its generalization to higher target counts is evaluated in scenarios with seven to ten target ships, with 200 scenarios for each target count. Including the own ship, the vessel density increases from $0.0273L^{-2}$ with six target ships to $0.0430L^{-2}$ with ten. Over the same range, the mean number of target trajectories with scheduled interaction points within $2.5\,\mathrm{m}$ of the nominal own ship route increases from 5.62 to 8.03, while the mean maximum number of scheduled interaction points within a $5\,\mathrm{s}$ window increases from 3.53 to 5.76. These values characterize a compact multi-ship encounter setting with increasing interaction density and temporal overlap. The USV hydrodynamic parameters and training settings are listed in Tables~\ref{tab:hydro_params} and~\ref{tab:training_params}, respectively. Table~\ref{tab:execution_params} summarizes the fixed parameters of the proposed CoReCBF-QP and the COLREGs-aware reference used in these evaluations.

\begin{table}[htbp]
  \centering
  \caption{Hydrodynamic parameters of the USV.}
  \label{tab:hydro_params}
  \begin{tabularx}{\linewidth}{
  >{\raggedright\arraybackslash}l
  >{\centering\arraybackslash}l
  >{\centering\arraybackslash}X
}
    \hline
    Parameters & Value & Description \\
    \hline
    $L$ & 2.0 & Length of the hull (m) \\
    $B$ & 1.08 & Beam of the hull (m) \\
    $m_{11}$ & 19.0 & Surge mass (kg) \\
    $m_{22}$ & 35.2 & Sway mass (kg) \\
    $m_{33}$ & 4.2 & Yaw inertia ($kg\cdot m^2$) \\
    $d_{u}, d_{u2}$ & 4.0, 10.0 & Linear and quadratic drag in surge \\
    $d_{v}, d_{v2}$ & 1.0, 2.0 & Linear and quadratic drag in sway \\
    $d_{r}, d_{r2}$ & 10.0, 15.0 & Linear and quadratic drag in yaw \\
    \hline
  \end{tabularx}
\end{table}

\begin{table}[htbp]
  \centering
  \caption{Relevant parameters used in training.}
  \label{tab:training_params}
  \begin{tabularx}{\linewidth}
{
  >{\raggedright\arraybackslash}l
  >{\centering\arraybackslash}l
  >{\centering\arraybackslash}X
}
    \hline
    Parameter Name & Value & Description \\
    \hline
    Total timesteps & 5,000,000 & Total sampling steps \\
    Batch size & 2048 & Size of the batch \\
    Number of steps & 1024 & Steps per update \\
    Learning rate & $3\times 10^{-4}$ & Learning rate of NN \\
    Gamma & 0.99 & Discount factor \\
    Epochs & 4 & Number of epochs \\
    Clip range & 0.15 & Clipping parameter \\
    Entropy coefficient & 0.02 & Coefficient of entropy \\
    Value function coefficient & 0.5 & Coefficient of value function \\
    Target KL & 0.03 & Target KL divergence \\
    \hline
  \end{tabularx}
\end{table}

\begin{table}[htbp]
  \centering
  \caption{{Fixed parameters used in CoReCBF-QP execution.}}
  \label{tab:execution_params}
  {
  \begin{tabularx}{\linewidth}{
    >{\raggedright\arraybackslash}X
    >{\centering\arraybackslash}l
  }
    \hline
    Parameter group & Value \\
    \hline
    Encounter boundaries $(\theta_h,\theta_{\lim})$
      & $(5^\circ,112.5^\circ)$ \\
    CPA reference gate $(d_{\mathrm{safe}},T_{\mathrm{CPA}})$
      & $(3.0\,\mathrm{m},10.0\,\mathrm{s})$ \\
    Transition slopes $(\kappa_\beta,\kappa_t)$
      & $(60,60\,\mathrm{s}^{-1})$ \\
    Yaw-reference parameters $(\alpha_r,\eta_{\mathrm{ref}})$
      & $(0.5\,\mathrm{rad\,s^{-2}},0.1875)$ \\
    CoReCBF and QP parameters $(R_i,k_H,\mathbf W)$
      & $(2.2\,\mathrm{m},1.0\,\mathrm{s}^{-1},
      \operatorname{diag}(1,2))$ \\
    \hline
  \end{tabularx}
  }
\end{table}

\subsection{Robustness Verification Against Perception Uncertainty}
\label{subsec:robustness}

\begin{table*}[t]
\centering
\caption{{Aggregate performance comparison across scenarios with three to six target ships.}}
\label{tab:robustness_obs3_6}
{
\footnotesize
\setlength{\tabcolsep}{2.0pt}
\renewcommand{\arraystretch}{1.15}
\resizebox{\textwidth}{!}{%
\begin{tabular}{lccccccc}
\toprule
Method
& \shortstack{Success rate\\(\%) $\uparrow$}
& \shortstack{Collision rate\\(\%) $\downarrow$}
& \shortstack{Timeout rate\\(\%) $\downarrow$}
& \shortstack{Mean minimum\\distance ($\mathrm{m}$) $\uparrow$}
& \shortstack{Average speed\\($\mathrm{m\,s^{-1}}$) $\uparrow$}
& \shortstack{Average path length\\($\mathrm{m}$) $\downarrow$}
& \shortstack{Latency\\($\mathrm{ms}$) $\downarrow$} \\
\midrule
DRL-VO
& 76.450 (2.692)$^{**}$
& 23.500 (2.685)$^{**}$
& \textbf{0.050 (0.068)}
& \textbf{3.063 (0.235)}
& 1.243 (0.005)$^{*}$
& 32.116 (0.434)$^{*}$
& 1.08 (0.02) \\

MSE-PPO
& 79.725 (11.605)$^{*}$
& 20.200 (11.614)$^{*}$
& 0.075 (0.168)
& 2.683 (0.130)$^{**}$
& 1.364 (0.074)
& 31.154 (0.495)
& \textbf{0.91 (0.09)} \\

Lag-U
& 81.950 (2.922)$^{*}$
& 17.725 (3.147)$^{*}$
& 0.325 (0.301)
& 2.733 (0.084)$^{*}$
& \textbf{1.411 (0.032)}
& 31.310 (0.343)
& 3.80 (0.01) \\

Hetero-PPO
& 86.125 (12.620)
& 13.725 (12.426)
& 0.150 (0.224)
& 2.637 (0.160)$^{*}$
& 1.323 (0.088)
& 30.548 (0.629)
& 0.92 (0.09) \\

Proposed CWVL
& \textbf{91.375 (4.239)}
& \textbf{8.450 (4.284)}
& 0.175 (0.209)
& 2.923 (0.153)
& 1.271 (0.021)
& 31.177 (0.701)
& 0.93 (0.05) \\

COLREGs-MPCC
& 33.625$^{**}$
& 61.250$^{**}$
& 5.125$^{**}$
& 2.696$^{*}$
& 0.776$^{**}$
& \textbf{28.371}
& 298.20 \\
\bottomrule
\end{tabular}
}

\vspace{0.5ex}
\begin{minipage}{0.99\textwidth}
\scriptsize
MSE-PPO denotes the PPO baseline with a matched implementation and a conventional critic with a single output optimized by mean-squared error. Results for the learning methods are reported as mean (SD), where SD is the sample standard deviation across five independently trained seeds, whereas COLREGs-MPCC is reported as a deterministic point estimate. Boldface indicates the best result for each metric. * and ** indicate that the proposed CWVL significantly outperformed the corresponding baseline at $p<0.05$ and $p<0.01$, respectively.
\end{minipage}
}
\end{table*}

{ To evaluate robustness to perception uncertainty within the training distribution, we compared the proposed CWVL with five baselines in scenarios containing three to six target ships. These comprised four learning baselines, MSE-PPO, Hetero-PPO, DRL-VO, and Lag-U, together with the optimization controller COLREGs-MPCC. For each scenario, a single mismatch window of 30 steps was scheduled from candidate steps grouped by path progress and derived from the replayed target trajectories. At each realized step within the window, the true measurement covariance was set to \(100R\), and the measurement was generated from the target state 20 steps earlier, while the Kalman filter continued to use the nominal covariance \(R\). With \(\Delta t=0.1\,\mathrm{s}\), the scheduled window and measurement delay corresponded to 3.0 s and 2.0 s, respectively. The five learning methods were evaluated using five independently trained seeds per method and 200 scenarios for each number of target ships, yielding 800 evaluation episodes per seed. The deterministic COLREGs-MPCC controller was evaluated once on the corresponding suite of 200 scenarios for each number of target ships. Success, collision, and timeout rates and mean minimum distance were aggregated over all evaluated episodes; speed and average path length were calculated over successful episodes; and online control latency was aggregated over executed control steps. Significance was assessed separately for each performance metric with Bonferroni correction.}

{ Table~\ref{tab:robustness_obs3_6} summarizes the aggregate results across scenarios containing three to six target ships. Among the learning methods, the proposed CWVL achieved the highest mean success rate of \(91.38\%\), followed by Hetero-PPO at \(86.13\%\), Lag-U at \(81.95\%\), MSE-PPO at \(79.73\%\), and DRL-VO at \(76.45\%\). CWVL achieved significantly higher success rates and lower collision rates than MSE-PPO, DRL-VO, and Lag-U, together with significantly larger mean minimum distances than MSE-PPO, Hetero-PPO, and Lag-U. It also significantly improved speed and average path length relative to DRL-VO, although its speed was lower than those of MSE-PPO, Hetero-PPO, and Lag-U. Its online control latency remained comparable to MSE-PPO and Hetero-PPO and lower than that of DRL-VO and Lag-U. In comparison with the optimization controller COLREGs-MPCC, CWVL achieved significantly higher success rate, speed, and mean minimum distance, together with a significantly lower collision rate and lower online control latency, whereas COLREGs-MPCC yielded the shortest average path length.}

\begin{figure*}[!t]
    \centering
    \includegraphics[width=0.94\textwidth]{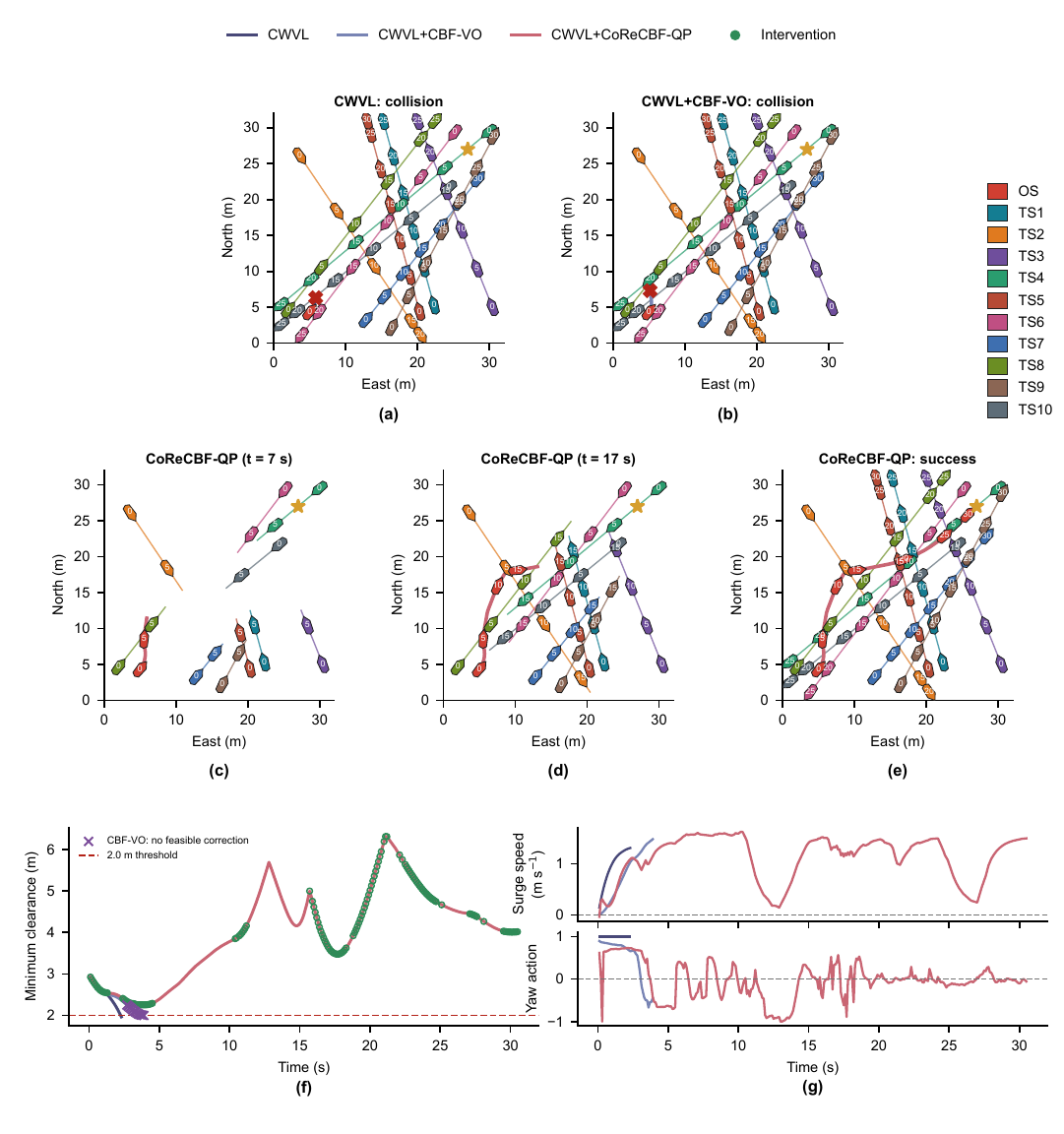}
    \caption{{Representative trajectories and control responses in a high-interaction encounter with ten target ships. (a,b) Trajectories without a safety layer and with CBF-VO, both terminating in collision. (c--e) The proposed CoReCBF-QP trajectory at three key stages: $t=7\,\mathrm{s}$, $t=17\,\mathrm{s}$, and goal arrival at $t=30.5\,\mathrm{s}$. Numerical labels mark vessel positions at intervals of $5\,\mathrm{s}$, and panels (c--e) additionally show the current states. (f) Minimum clearance, where the dashed line denotes the $2.0\,\mathrm{m}$ threshold, green circles denote CoReCBF-QP interventions, and purple crosses denote steps without a feasible CBF-VO correction. (g) Surge speed and yaw action under the three execution configurations.}}
    \label{fig:corecbf_case}
\end{figure*}

\begin{figure*}[!t]
    \centering
    \includegraphics[width=0.94\textwidth]{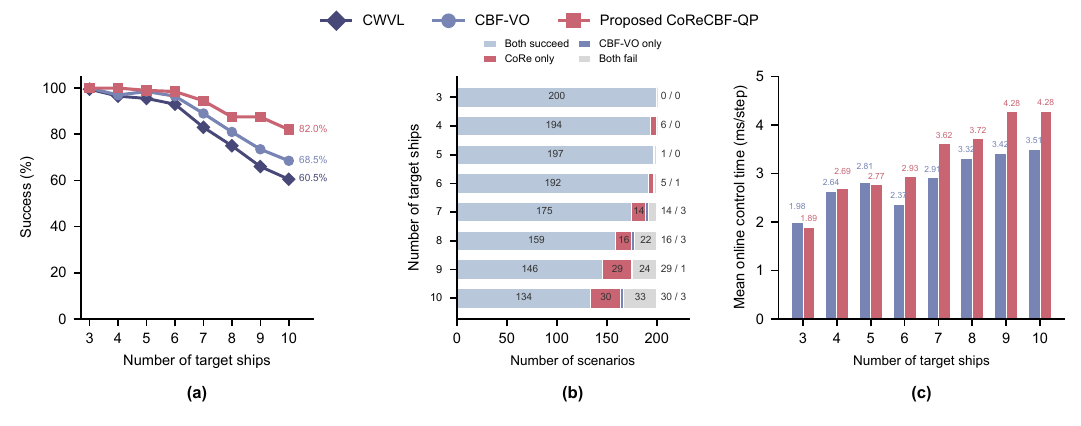}
    \caption{{Generalization of the proposed CWVL policy across different numbers of target ships under three execution configurations. (a) Success rates for CWVL, CWVL+CBF-VO, and CWVL with the proposed CoReCBF-QP. Labels at ten target ships report the corresponding success rates. (b) Outcome comparison at the scenario level between CBF-VO and the proposed CoReCBF-QP. The values to the right report successes unique to the proposed CoReCBF-QP and CBF-VO, respectively. (c) Mean online control time for CBF-VO and the proposed CoReCBF-QP. Values above the bars report the corresponding mean times.}}
    \label{fig:corecbf_count_generalization}
\end{figure*}

\subsection{{CoReCBF-QP in Encounters with Dense Interactions}}
\label{subsec:corecbf_high_interaction}
{
To examine the proposed framework in a compact encounter with dense interactions, we use a reproducible scenario generated from Expected Encounter Points along the nominal USV route. The initial kinematic states of ten target ships are calculated backward from the sampled encounter points while preserving kinematic consistency. This construction increases the occurrence of crossing and head-on interactions. The same scenario is evaluated using the proposed CWVL policy without a safety layer, with CBF-VO, and with the proposed CoReCBF-QP. The initial configuration is given in Table~\ref{tab:initial_states}, where the USV navigates from $(5.00,4.30)$ to the goal at $(27.00,27.00)$.
}

\begin{table}[hbt!]
{
\centering
\caption{Initial states in the ten-target high-interaction encounter.}
\label{tab:initial_states}
\begin{tabularx}{\linewidth}
{
  >{\raggedright\arraybackslash}X
  >{\centering\arraybackslash}X
  >{\centering\arraybackslash}X
}
  
\toprule
Ship number & Initial position & Orientation (deg) \\
\midrule
USV  & $(5.00, 4.30)$   & 50.27  \\
TS1  & $(22.30, 5.22)$  & 105.36 \\
TS2  & $(3.60, 26.12)$  & 303.95 \\
TS3  & $(30.44, 4.98)$  & 111.92 \\
TS4  & $(29.98, 29.49)$ & 219.75 \\
TS5  & $(20.29, 4.35)$  & 105.34 \\
TS6  & $(25.20, 29.35)$ & 232.89 \\
TS7  & $(12.82, 3.24)$  & 51.41  \\
TS8  & $(2.01, 4.61)$   & 52.19  \\
TS9  & $(16.44, 2.16)$  & 61.97  \\
TS10 & $(24.39, 21.68)$ & 219.25 \\
\bottomrule
\end{tabularx}
}
\end{table}
{
Figure~\ref{fig:corecbf_case} first shows the collision processes without a safety layer and with CBF-VO. Without a safety layer, TS8 approaches from the port quarter as the clearance decreases from $2.93\,\mathrm{m}$ at $t=0.1\,\mathrm{s}$ to $1.96\,\mathrm{m}$ at $t=2.3\,\mathrm{s}$. CBF-VO initially reduces the surge and port-turn commands, but ceases to provide a feasible correction from $t=2.7\,\mathrm{s}$ as the encounter continues to close. Collision with TS8 occurs at $t=3.9\,\mathrm{s}$, with a minimum distance of $1.9987\,\mathrm{m}$.
}

{
With the proposed CoReCBF-QP, the constraint associated with TS8 is activated at $t=0.1\,\mathrm{s}$. The resulting surge moderation and starboard correction increase the passing distance to $2.25\,\mathrm{m}$ at approximately $t=3.9\,\mathrm{s}$. By the first snapshot at $t=7\,\mathrm{s}$, the OS has cleared this initial conflict and overtakes TS7 ahead in accordance with Rule~13.
}

{
The second snapshot captures the subsequent crossing and head-on interactions. When TS5 appears on the starboard side at $t=10.4\,\mathrm{s}$, the QP changes the policy's slight port turn into a starboard correction while maintaining forward surge, following the give-way maneuver specified by Rule~15. TS1 enters the crossing area at $t=10.7\,\mathrm{s}$, and the starboard correction is strengthened at $t=10.8\,\mathrm{s}$ to clear TS5 and TS1. When the head-on vessel TS4 approaches, the controller reduces the conflicting port-turn request at $t=11.3\,\mathrm{s}$ and maintains the starboard maneuver required by Rule~14. The snapshot at $t=17\,\mathrm{s}$ shows the resulting avoidance path. The OS subsequently adjusts surge and yaw to clear TS7 at $t=18.7\,\mathrm{s}$, followed by a starboard correction to avoid TS9 at $t=25.1\,\mathrm{s}$ in accordance with Rule~15. After completing these successive maneuvers, the USV reaches the goal at $t=30.5\,\mathrm{s}$. The three snapshots illustrate the evolution of the COLREGs-aware avoidance trajectory.
}

\subsection{Comparative Analysis and Discussion}
\label{sec:comparative_analysis}

{
We evaluate the proposed framework through two complementary experiments. First, we examine the robustness provided by credibility-weighted value learning under scheduled mismatch in state estimation. CWVL is compared with five baselines based on learning and optimization in its training range of three to six target ships. Second, we examine execution performance in scenarios containing three to ten target ships using three configurations of the CWVL policy: without a safety layer, with CBF-VO, and with the proposed CoReCBF-QP. This comparison evaluates CoReCBF-QP across all numbers of target ships, while scenarios containing seven to ten target ships further assess generalization beyond the training range in denser encounters involving multiple ships.
}

{
For the comparison across different numbers of target ships, we report success, speed, path length, COLREGs event compliance, and mean online control time. COLREGs event compliance is evaluated over Rule~14 head-on and Rule~15 give-way crossing events. An event is compliant when the maneuver is timely and substantial, maintains safe separation, resolves the encounter, and produces a valid passing side.
}

{
Across scenarios with three to ten target ships, the aggregate success rates were \(83.63\%\), \(88.00\%\), and \(93.63\%\) without a safety layer, with CBF-VO, and with the proposed CoReCBF-QP, respectively. The corresponding average speeds were \(1.180\), \(1.173\), and \(1.167\,\mathrm{m\,s^{-1}}\). The mean online control times were \(2.897\) and \(3.362\,\mathrm{ms}\) per step for CBF-VO and the proposed CoReCBF-QP.
}

{
Integrating the proposed CWVL policy with the proposed CoReCBF-QP also improved COLREGs event compliance. Across scenarios containing three to ten target ships, the pooled compliance rate increased from \(48.98\%\) to \(52.26\%\). With ten target ships, the compliance rate increased from \(47.51\%\) to \(53.04\%\), corresponding to a gain of \(5.54\) percentage points.
}
\subsection{Generalization Across Numbers of Target Ships}
\label{subsec:count_generalization}

{
Figure~\ref{fig:corecbf_count_generalization} compares three execution configurations of the proposed CWVL policy across scenarios containing three to ten target ships. Scenarios containing three to six target ships cover the training range, while those containing seven to ten target ships evaluate generalization across different numbers of target ships under increasingly complex encounters involving multiple ships. Each number of target ships comprises 200 scenarios.
}

{
Across scenarios containing three to ten target ships, the proposed CoReCBF-QP matched or exceeded the success rate of both alternatives for every number of target ships, with an aggregate advantage of \(5.63\) percentage points over CBF-VO. The improvement remained pronounced beyond the training range, reaching an \(82.0\%\) success rate with ten target ships, compared with \(68.5\%\) for CBF-VO. The corresponding mean online control times were \(3.362\) and \(2.897\,\mathrm{ms}\) per step, both well below the \(100\,\mathrm{ms}\) control period.
}

\section{Conclusion and Future Work}
\label{section:Conclusion and future work}
{
This study presented a unified credibility-aware safe reinforcement learning framework for USV navigation under perception uncertainty. The proposed CWVL improves value learning under mismatch in state estimation by weighting critic updates according to filter credibility. The proposed CoReCBF-QP expands the collision geometry under uncertainty and incorporates braking and turning recovery into the barrier constraints. At execution, a continuous COLREGs-aware reference promotes the starboard maneuvers associated with Rule~14 head-on and Rule~15 give-way crossing encounters. Experiments demonstrate the robustness of CWVL to scheduled mismatch in state estimation. Across scenarios containing three to ten target ships, the proposed CoReCBF-QP matched or exceeded the success rate of both alternatives for every number of target ships. This performance extended beyond the training range to scenarios containing seven to ten target ships, reaching an $82.0\%$ success rate with ten target ships. Combining the proposed CWVL with the proposed CoReCBF-QP also increased COLREGs event compliance across the same range of target ship numbers, including a gain of $5.54$ percentage points with ten target ships. Future work will consider dynamic hydrodynamic disturbances and non-Gaussian belief representations, together with target ships undergoing coordinated turns, acceleration, and deceleration.
}

\bibliographystyle{IEEEtran}
\bibliography{references}
\end{document}